\documentclass[sigconf]{aamas} 

\usepackage{balance} % for balancing columns on the final page
\usepackage{newfloat}
\usepackage{paralist}
\usepackage{listings}
\usepackage{algorithm}
\usepackage{algorithmic}
\usepackage{enumitem}
\setcopyright{ifaamas}
\acmConference[AAMAS '23]{Proc.\@ of the 22nd International Conference
on Autonomous Agents and Multiagent Systems (AAMAS 2023)}{May 29 -- June 2, 2023}
{London, United Kingdom}{A.~Ricci, W.~Yeoh, N.~Agmon, B.~An (eds.)}
\copyrightyear{2023}
\acmYear{2023}
\acmDOI{}
\acmPrice{}
\acmISBN{}

\acmSubmissionID{254}

\title{	Learning Density-Based Correlated Equilibria for Markov Games}

\author{Libo Zhang$^{1,2,*}$, Yang Chen$^{2,*,\dag}$, Toru Takisaka$^1$,  Bakh Khoussainov$^1$, Michael Witbrock$^2$, \\ and Jiamou Liu$^{2,\dag}$}
\affiliation{\institution{$^1$ University of Electronic Science and Technology of China}
\country{China}}
\email{{takisaka, bmk}@uestc.edu.cn}
\affiliation{
  \institution{$^2$ The University of Auckland}
  \country{New Zealand}}
\email{lzha797@aucklanduni.ac.nz, {yang.chen,jiamou.liu,m.witbrock}@auckland.ac.nz}

\begin{abstract}
Correlated Equilibrium (CE) is a well-established solution concept that captures coordination among agents and enjoys good algorithmic properties. In real-world multi-agent systems, in addition to being in an equilibrium, agents' policies are often expected to meet requirements with respect to safety, and fairness. Such additional requirements can often be expressed in terms of the {\em state density} which measures the state-visitation frequencies during the course of a game. %It is thus crucial to find a set of CEs satisfying particular state-density requirements. 
However, existing CE notions or CE-finding approaches cannot explicitly specify a CE with particular properties concerning state density; they do so implicitly by either modifying reward functions or using value functions as the selection criteria. The resulting CE may thus not fully fulfil the state-density requirements. In this paper, we propose {\em Density-Based Correlated Equilibria} (DBCE), a new notion of CE that explicitly takes state density as selection criterion. Concretely, we instantiate DBCE by specifying different state-density requirements motivated by real-world applications. To compute DBCE, we put forward the {\em Density Based Correlated Policy Iteration} algorithm for the underlying control problem. We perform experiments on various games where results demonstrate the advantage of our CE-finding approach over existing methods in scenarios with state-density concerns.
\end{abstract}


\newcommand{\BibTeX}{\rm B\kern-.05em{\sc i\kern-.025em b}\kern-.08em\TeX}

\usepackage{microtype}
\usepackage{subcaption}
\usepackage{booktabs} % for professional tables
\usepackage{multirow}
\usepackage{algorithm}
\usepackage{algorithmic}
\usepackage{bm}

\usepackage{amssymb}
\usepackage{mathtools}
\usepackage{paralist}

\newcommand{\beq}[1][\vspace{0.3em}]{#1\begin{equation}}
\newcommand{\eeq}{\end{equation}}

\newcommand{\Qv}[0]{{{\bf Q}}}

\newcommand{\Amc}[0]{{{\mathcal{A}}}}

\newcommand{\Smc}[0]{{{\mathcal{S}}}}

\newcommand{\av}[0]{{{\bm a}}}

\newcommand{\rv}[0]{{{\bm r}}}

\newcommand{\piv}[0]{{\bm{\pi}}}

\newcommand{\regret}{\mathsf{reg}}
\newcommand{\pol}{\bold{\Pi}}
\newcommand{\real}{\mathbb{R}}
\newcommand{\nat}{\mathbb{N}}
\newcommand{\BFerr}{\mathsf{BFError}}
\newcommand{\Abmc}{\bm{\Amc}}

\newcommand{\MaxBF}{\mathsf{MaxBF}}
\newcommand{\MaxRegret}{\mathsf{MaxReg}}

\usepackage{amsthm} 
\newtheorem{theorem}{Theorem}%[section]
\newtheorem{lemma}{Lemma}

\newtheorem{assumption}{Assumption}
\newtheorem{problem}{Problem}
\newtheorem{definition}{Definition}
\usepackage[disable]{todonotes}
\begin{document}

%%% The following commands remove the headers in your paper. For final 
%%% papers, these will be inserted during the pagination process.

\pagestyle{fancy}
\fancyhead{}

%%% The next command prints the information defined in the preamble.

\maketitle 

%%%%%%%%%%%%%%%%%%%%%%%%%%%%%%%%%%%%%%%%%%%%%%%%%%%%%%%%%%%%%%%%%%%%%%%%
\renewcommand{\thefootnote}{}
\footnotetext{* Equal contributions\\
$\dag$ Corresponding author}
\section{Introduction}

% MAS, Markov game, why do we use CE
%Solving games has drawn interest from research communities in the computer science domain. 
%Finding optimal policies of agents is a central question in the study of  multi-agent systems.  
A central question in the study of multi-agent systems is finding policies for rational game players to reach a particular form of equilibrium. 
A more recent trend in the investigation of this question is to incorporate policies' side effects \cite{klassen2022ai}. 
Indeed, in many real-world scenarios, it is difficult to define a reward function that captures all aspects of desired outputs of the agents. For example, a robotic system may gain a high reward by performing a specific risky manoeuvre that is less-than-desirable or engaging in actions that are seen as unethical \cite{schiff2020s}. When finding policies for agents, it is therefore not optimal to simply enable agents to achieve the highest possible rewards. Still, more importantly, the procedures must also satisfy other desirable properties, such as safety and fairness, that are not reflected by rewards. 

For simplicity, we formulate this type of problem as taking an $N$-player Markov game as input while asking for policies of agents that satisfy two types of requirements: 
\begin{enumerate}[leftmargin=*]
    \item {\bf\em Reward Requirement:} First, we expect that  the agents, being rational, will not unilaterally deviate from their policies due to utility concerns, and
    \item {\bf\em Non-reward Requirements:} Then, the policies must satisfy certain non-utility-based requirements that confine the runs of the multi-agent system.
\end{enumerate}

In this paper, we study an instance of the general problem above. (1) For the reward requirement, we specify a solution concept that factors into the possible coordination among agents. More specifically, we adopt \emph{correlated equilibrium} (CE) \cite{hart2001reinforcement} as the solution concept. Compared to Nash equilibrium (NE), widely adopted in this field \cite{holt2004nash}, CE does not require independence among agents and is suitable for a wider range of practical scenarios. Moreover, the set of CEs constitutes a convex polytope. Therefore, it is easy to compute via linear programming. Many adaptive procedures are shown to converge to CE rather than the more restricted NE \citep{hart2000simple,gordon2008no}. 
%
%Second, 
%
%
% Requirements rather than CE
%
(2) For the non-reward requirements, we consider essential properties which can be loosely translated to, e.g., {\it ``certain situation should not take place''}, {\it ``certain situation should happen with a prescribed frequency''} and {\it ``two situations should happen with the same frequency''}.
%
%However, although CE captures the social profit in the concept, it does not include any other property, avoiding side effects. 
%
These properties can be captured by examining a {\em run}, i.e., the sequence of states the agents are in during the game. 
More precisely, they are  {\it state-distribution requirements} that are defined in terms of the visitations to states in the game:
\begin{itemize}[leftmargin=*]
    \item \emph{Safety requirements.} These conditions demand that {\it certain bad states should not be visited}. Many industrial applications involve dangerous states that should never happen. Take, e.g., the low-power status of a robot system \cite{qin2021density}. Similar concerns can happen from an ethical perspective as well \cite{schiff2020s}.
    \item \emph{Frequency requirements.} These conditions demand that {\it certain states should be visited with a fixed frequency}. To generalise safety, the system may be expected to visit certain states with a certain proportion in the long term. For example, we may want a robotic system to run in high-efficiency mode 30\% of the time, and the rest 70\% time in normal mode.
    \item \emph{Fairness requirements.} These conditions demand that {\it two states should be visited with the same frequency}. One may also wish to balance the visitation frequency of two different states of the system for the sake of, e.g., system stability. For instance, if there are two charging stations for a team of uncrewed aerial vehicles, one may wish to balance their use rate to avoid unnecessary queuing. Or at a crossroads, traffic lights in two directions should be green with equal frequency.
\end{itemize}

% How to capture by density
The {\em state density function} may be employed %has been well justified and 
 to measure the state visitation frequency when navigating the environment using a policy \cite{qin2021density}.
%defined as $\rho^\piv(s) = \sum_{t=0}^\infty \gamma^t Pr(s^t=s|\piv,s_0 \sim \eta)$ in which $s$ is a state, $t$ is the time step, $\piv$ is a policy, $\gamma$ is the discounted factor and $\eta$ is the initial state-distribution. 
The function can express the aforementioned state-distribution requirements. However, so far, no work on CE or CE-finding algorithms has {\em explicitly} incorporated requirements defined by the state density function. 
%If we consider the state-distribution requirements as visitation to some states, density functions can be used to describe the requirements by discounted visitation frequency.
On the other hand, methods have been introduced to {\em implicitly} express these non-reward requirements by imposing additional constraints on rewards.  Yet, these methods may not be sufficient to meet these desired requirements. In detail, the existing techniques fall into two categories:
\begin{enumerate}[leftmargin=*]
    \item {\it Risk-sensitive reward modification.} This method adds additional terms that tweak the reward structure, such as negative rewards for undesired states or imposing variance as risk-terms \cite{borkar2002q,mihatsch2002risk,geibel2005risk,shen2014risk}. This method has been preferred when the target is simple, and a tweaking strategy can be efficiently designed. However, it requires parameter fine-tuning as the optimal policy is sensitive to reward settings. When computing CE, this method changes the shape of the CE set of the original game. As a result, the optimal policy found in a modified game may not be a CE to the original game. Moreover, designing a reward modification for complex targets such as {\it value-targeting requirement} requires expert domain knowledge, which is challenging in real-world applications.
    \item {\it Constrained methods.} This method directly takes the subset of policies by adding explicit constraints such as {\em constrained Markov games} \cite{altman2000constrained}. %It's a general idea, so it has been adapted to the tasks in which constraints can represent the requirements. 
    However, it requires parameter fine-tuning because the threshold for constraints can directly impact the game's performance and feasibility. Before solving the game, the optimal solution is invisible to the designers so setting a correct threshold is challenging. Additionally, when computing  CE, the introduced constraints may reduce the size of the CE set of the original game.
\end{enumerate}
Both methods above may change the shape or size of the feasible CE set of the original game. Such changes are illustrated  in Fig.~\ref{fig:FeaSet}

\begin{figure}
 \centering
  \includegraphics[width=\columnwidth]{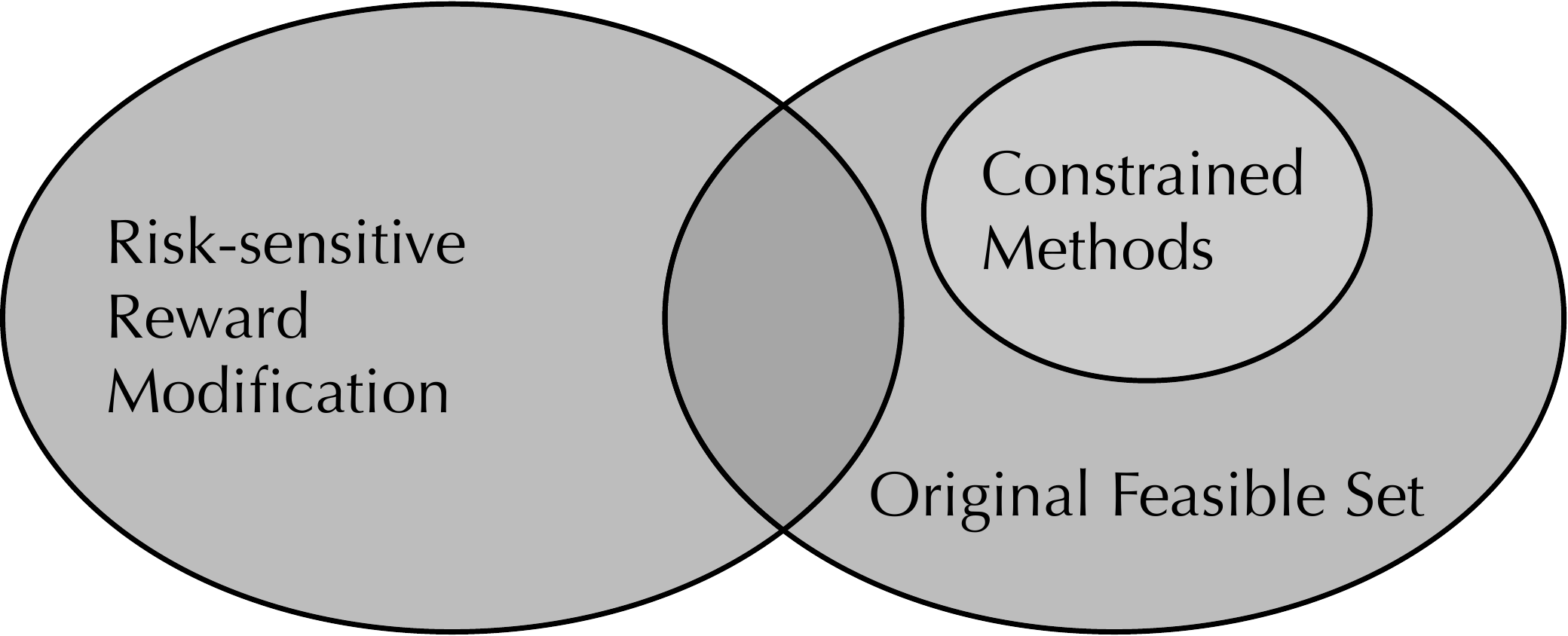}
\caption{This diagram shows the changes to the feasible CE sets by the two existing methods. A constrained method leads to a size-reduced feasible set, which can be empty if additional constraints are infeasible. The risk-sensitive reward modification generates a new game; thereby, the feasible set may be shifted. Moreover, an improper reward modification may cause an empty intersection between the feasible sets for the modified game and the original game.}
\label{fig:FeaSet}
\end{figure}

% DBCE + CONTRIBUTION
In this paper, we introduce \emph{Density-Based Correlated Equilibria} (DBCE) in the context of Markov games. By using density functions as a selection criterion, DBCE explicitly integrates state-distribution requirements (non-reward requirements) and reward requirements to a novel CE concept without suffering the issues above in the existing CE notions or CE-finding approaches.

However, having an equilibrium concept does not necessarily imply an effective way to find it. Directly computing a DBCE is intractable due to the inconsistency between the measurements of a policy's state density and cumulative rewards, preventing us from optimising the two in the same space. To settle this challenge, we employ the notion of {\em occupancy measure}, i.e., the cumulative state-action visitation frequency, in terms of which both the state density and cumulative rewards can be represented; it thus allows us to optimise the two in a unified fashion. This machinery gives rise to our proposed algorithm for computing DBCE named {\em Density Based Correlated Policy Iteration} (DBCPI). More specifically, DBCPI runs in such an iterative manner that alternates between the update of agents' policies and the occupancy measure: the policies are updated by finding a CE that is induced by the current occupancy measure and satisfies the non-reward requirements; the current occupancy is subsequently updated in accordance with the updated policy. Moreover, we provide a theoretical justification for DBCPI where the convergence conditions are given.

Our primary contributions are summarised as follows: 
\begin{enumerate}[leftmargin=*]
\item We propose a new CE concept for Markov games-- Density-Based Correlated Equilibria (DBCE) -- which exploits the state density function to explicitly capture non-reward requirements without changing the set of all feasible CEs.
\item To compute DBCE, we come up with Density-Based Correlated Policy Iteration (DBCPI). We show that under certain assumptions, this mechanism converges to a valid DBCE.
\item We test DBCPI against existing approaches on different simulated scenarios motivated by real-world applications. Experimental results demonstrate our machinery's advantage in finding CE with those mentioned above, additional non-reward requirements, i.e., safety, frequency and fairness.
\end{enumerate}

\section{Related Works}
\subsubsection*{Equilibrium Concepts.}

In multi-player games, especially non- cooperative games, solving a game amounts to finding an equilibrium. Nash Q-learning \cite{hu2003nash} extended the canonical Q-learning to general-sum Markov games to find Nash equilibrium. As for finding CEs, some work \cite{murray2007finding, Dermed2009solving} attempted to calculate the whole set of CEs by determining or approximating the boundary of the resulting expected-reward space, which has been shown to be a convex polytope. Alternatively, \citet{greenwald2003correlated} proposed a Q-learning-like algorithm to find an instance of CE in a Markov game rather than the whole set.
%\subsubsection*{Correlated Equilibria.}
%CE \cite{aumann1974subjectivity} provided a more general solution than Nash Equilibrium. Benefiting from the correlation among agents, 
Another line of work focuses on exploiting the application value of CEs in real-world scenarios, such as \citet{yu2014multi} and \citet{han2007distributive} used CEs to coordinate equipment in industrial scenarios; \citet{jin2019dynamic} used CEs as the solution to a outsource task pricing problem. A series of work captures some particular properties by selecting a special subclass of CEs from the entire set. For instance, %\citet{han2007distributive} forms the radio resource allocation problem as CE selection problem; 
\citet{ortiz2007maximum} and \citet{ziebart2010maximum} choose the CE with the maximum policy entropy to ensure the uniqueness of the solution to a game.

\subsubsection*{Non-reward Requirements.}
In real-world applications, the non-reward requirements are inevitable. Some work adopts an implicit way to satisfy these requirements by modifying the reward functions  \cite{garcia2015comprehensive}. One popular method is to augment the reward function with risk-sensitive terms such as variance~\cite{markowitz1952portfolio} and exponential utility function~\cite{chung1987discounted}. Rather than implicit reward tweaking, logic instruction \cite{hasanbeig2018logically} that explicitly describes the goal is also considered as one method to modify the reward function. Recently, reinforcement learning algorithms with different risk-sensitive factors have been studied in various aspects \cite{borkar2002q,mihatsch2002risk,geibel2005risk,shen2014risk}. 
Some work along this line adds constraints to the learned policy in order to capture safety concerns: \citet{altman1993asymptotic} studied constrained MDP, and subsequently, Q-learning was extended to constrained MDP by \cite{gattami2019reinforcement}. Constrained method was later further extended to Markov Games in the multi-agent setting \cite{altman2000constrained,jiang2020finding,ge2020multi,altaian2007cons}. Among these constraints, the state density stands out as a particular one. Typical work includes \cite{geibel2005risk} that directly specified the unwanted states to avoid getting in, and \citet{qin2021density} proposed to use density functions as constraints to guide the finding of an optimal policy in reinforcement learning.

\section{Preliminaries}\label{sec:pre}
The set of all natural numbers, reals, and non-negative reals are denoted by $\nat$, $\real$, and $\real_{\geq 0}$, respectively. 
For a natural number $N$, the set $\{1,\ldots,N\}$ is denoted by $[N]$. 
%For ordered mathematical objects $\Xmc_1,\ldots,\Xmc_N$ with $N\in \nat$, 

\subsubsection*{Markov Games}  Markov games, also known as stochastic games, are extensions of Markov decision processes to the multi-agent setting, where a set of agents act in a stochastic environment, each aiming to maximise its cumulative rewards. %The reader is referred to, e.g., \citep{littman1994markov}, for more details on Markov games. 
\begin{definition}
    An $N$-agent {\em Markov game} is a tuple $$(\Smc, \{ \Amc_i \}_{i=1}^N, P, \{r_i\}_{i=1}^N, \eta, \gamma), \text{ where }$$  
    \begin{itemize}[leftmargin=*]
        \item $\Smc$ is the set of {\em states}, 
        \item $\Amc_i$ is the set of {\em actions} for the $i$th agent,  
        \item $P: \Smc \times \Abmc \to \Delta(\Smc)$ is the {\em transition function} that specifies the transition probability between two states given a {\em joint action} $\av = (a_1, \ldots, a_n)$, where $\Abmc = \times_{i=1}^N \Amc_i$ is the space of joint actions and $\Delta(\Smc)$ denotes the set of probability distributions over $\Smc$,
        \item $r_i: \Smc \times \Abmc \to \mathbb{R}$ is a {\em reward function} that determines agent $i$'s immediate reward of a joint action in a state,
        \item $\eta \in \Delta(\Smc)$ is the {\em initial distribution} of states,
        \item $\gamma\in (0,1)$ is a {\em discount factor}. 
    \end{itemize}
\end{definition}
Throughout, we use bold variables without subscripts to represent the concatenation of the corresponding variables for all agents and use the subscript $-i$ to denote all agents other than $i$, e.g., $\av = (a_1, \ldots, a_n) = (a_i, \av_{-i})$ denotes a {\em joint action} of all agents. 
\begin{definition}
    The agents' (stationary) {\em joint policy} is a function 
\[
\piv: \Smc \to \Delta(\Abmc)
\]
%%%T: we assume stationary, not Markovian. The latter means policies can depend on the time step as well as the current state. Cf. for example Def2.4 of https://books.google.co.jp/books?id=mYRyEAAAQBAJ
which specifies agents' probabilistic choice of actions according to the current state. The set of all joint policies is denoted by $\pol$.
\end{definition}
Each agent aims to find a policy to maximise its own {\em cumulative rewards} during the whole course of a game:
$\sum_{t=0}^\infty \gamma^t r_i(s^t, \av^t).$
%Throughout, we focus on {\em stochastic} and {\em Markovian} policies in the form of C, i.e., it chooses probabilistic actions and depends only on the current state. A {\em joint policy} $\piv: \Smc \to \Delta(\Abmc)$ maps from the state space to the joint action space. 
For each agent $i$, the {\em expected return} of a state-joint action pair under a joint policy $\piv$ is defined as: 
\beq\label{eq:q}
	Q_i^{\piv} (s, \av) \triangleq  \mathbb{E} \left[ \sum_{t=0}^\infty \gamma^t r_i (s^t, \av^t) \bigg\vert s^0 = s, \av^0 = \av, P,  \piv \right].\nonumber
\eeq

\subsubsection*{Correlated Equilibria.}
A solution to a Markov is called an equilibrium that amounts to a joint policy where no agent has an incentive to unilaterally deviate to gain rewards.
Two canonical equilibrium concepts stand out concerning assumptions on different degrees of the independence among agents' policies. The well-known \emph{Nash equilibrium} (NE) \cite{fink1964equilibrium} requires independence among the agents, i.e., $\bm{\pi} = \times_{i=1}^N \pi_i$ where $\pi_i: \Smc \to \Delta(\Amc_i)$ denotes the policy of an individual agent. In comparison, \emph{correlated equilibrium} (CE) \cite{aumann1987correlated} generalises NE by capturing the coordination among agents, which is more suitable for multi-agent systems where agents coordinate their actions. %Agents under a CE policy can coordinate their actions in order to get a higher expected return. 
Conceptually, agents are coordinated by a {\em correlation device} that recommends an action $a_i \in \Amc_i$ to each agent $i$, who is aware of all other agents' conditional distribution $\piv_{-i}(\av_{-i} \vert s,  a_i)$. %$\piv_{-i}(\av_{-i} \vert s,  a_i)$. 
To be in a CE, each agent has no incentive to disobey the recommendation, i.e., selecting an alternate action $a_i' \in \Amc_i$, called the {\em deviation action}. 
%The loss of choosing a deviation action is formalized by the notion of \emph{regret}, as follows:

\begin{definition}\label{def:CE}
	A \emph{correlated equilibrium} (CE) for a Markov game is a joint policy $\piv$ that satisfies:
	\beq\label{eq:ce_constraints}
	 \forall i \in [N], s \in \Smc, a_i, a_i' \in \Amc_i 
	 , \quad 
	 \regret_\piv(s,i,a_i,a'_i) \leq 0.
    \eeq
	Here, the \emph{regret} $\regret_\piv(s,i,a_i,a'_i)$  embodies the expected reward gain of shifting to a deviation action:
	\begin{equation*}
	    \regret_\piv(s,i,a_i,a'_i) \triangleq \mathbb{E}_{\av_{-i} \sim \piv_{-i}(\cdot \vert s, a_i)} \left[ Q_i^{\piv}(s, a_i', \av_{-i}) -  Q_i^{\piv}(s, a_i, \av_{-i}) \right].
	\end{equation*}
\end{definition}

% \begin{definition}[CE]\label{def:CE}
% 	A \emph{correlated equilibrium} (CE) for a Markov game is a joint policy $\piv$ such that for each agent $i \in [N]$ and each action $a_i \in \Amc_i$, it cannot obtain an expected gain over rewards through selecting a deviation action $a_i'$. For a given policy $\piv$, $s\in \Smc$ the \emph{regret} $\regret_\piv(s,i,a_i,a'_i)$ is defined as the following value:
%     \[\mathbb{E}_{\av_{-i} \sim \piv_{-i}(\cdot \vert s, a_i)} \left[ Q_i^{\piv}(s, a_i', \av_{-i}) -  Q_i^{\piv}(s, a_i, \av_{-i}) \right].\] So requirements of CE can be guaranteed with the following constraints:
% 	\beq\label{eq:ce_constraints}
% 	 \forall i \in [N], s \in \Smc, a_i, a_i' \in \Amc_i 
% 	 , \quad 
% 	 \regret_\piv(s,i,a_i,a'_i) \leq 0.
% 	 %\mathbb{E}_{\av_{-i} \sim \piv_{-i}(\av_{-i} \vert s, a_i)} \left[ Q_i^{\piv^{\CE}}(s, a_i', \av_{-i}) -  Q_i^{\piv^{\CE}}(s, a_i, \av_{-i}) \right] \leq 0 ~ \forall i \in [N], s \in \Smc, a_i, a_i' \in \Amc_i.
% \eeq
% \end{definition}

The general existence of NE \citep{fink1964equilibrium} implies the existence of CE. CE has nicer mathematical properties than NE in the sense that the constraints in Eq.~\eqref{eq:ce_constraints} define an $N$-dimension polytope in agent's expected returns while the set of NE consists of isolated points \citep{neyman1997correlated} in the polytope. Consequently, the set of CEs for normal-form games (equivalent to one-shot Markov games) can be derived using linear programming as Eq.~\eqref{eq:ce_constraints} is a system of linear inequalities. 
Still, exactly computing CE for Markov games is generally intractable due to two reasons: (i) the constraints of CE turn to non-linear inequalities because both $Q$ and $\piv$ are unknown in Eq.~\eqref{eq:ce_constraints}; (ii) the number of corners of the CE polytope grows exponentially with the horizon increases \citep{ziebart2010maximum}.

%The constrains in Eq.~\eqref{eq:ce_constraints} define an $N$-dimension convex polytope of  For normal-form games (i.e., $T=1$), 

\subsubsection*{Density Functions.} 

A {\em density function} \citep{rantzer2001dual} $\rho: \Smc \to \mathbb{R}_{\geq 0}$
measures the visitation frequency of states when navigating the environment with a policy. Formally, for an infinite-horizon Markov game with its initial distribution $\eta$, discounted factor $\gamma$, the density function under a joint policy $\piv$ is defined as
\beq\label{eq:density function}
	\rho^\piv(s) \triangleq \sum_{t=0}^\infty \gamma^t \Pr(s^t = s \vert \piv, s^0 \sim \eta).\nonumber
\eeq

Notice that the density function can also be written in a recursive form as 
\beq\label{eq:density_recursive}
   \begin{aligned}
	&\rho^\piv(s) = \eta(s) + \piv(s,a)\gamma \sum_{s'\in \Smc} \sum_{\av \in \Abmc}  P(s \vert s', \av) \rho^\piv(s').\nonumber \\
	\end{aligned}
\eeq

\subsubsection*{Occupancy Measure.}
Similar to the density function, the \emph{occupancy measure} $\rho(s,a):\Smc \times \Abmc \to \mathbb{R}_{\geq 0}$ measures the visitation frequency of state-action pairs given a stationary policy. 
Formally, the occupancy measure $\rho^\piv$ under $\piv$ is defined as
\beq\label{eq:OccupancyMeas}
	\rho^\piv(s,\av) \triangleq \sum_{t=0}^\infty \gamma^t \Pr(s^t = s, \av^t = \av \vert \piv, s^0 \sim \eta).\nonumber
\eeq

We can calculate the density of a state by an equation $\rho^\piv(s) = \sum_{\av \in \Abmc} \rho^\piv(s,\av)$. 
Occupancy measure also has several properties %that are 
useful in policy synthesis via optimisation \cite{syed2008apprenticeship}. 
First, %it is known that 
a function $f:\Smc \times \Abmc \to \mathbb{R}$ is the occupancy measure under some stationary policy if and only if it satisfies the following \emph{Bellman flow (BF) constraints}:
%\beq\label{eq:density_recursive}
%    \begin{aligned}
%	&\rho^\piv(s,a) = \piv(s,a)\eta(s,a) + \piv(s,a)\gamma \sum_{s'\in \Smc} \sum_{\av \in \Abmc}  P(s \vert s', \av) \rho^\piv(s',a). \\
%	&~~~~~~~~~~~~~~~~~~~~~~~~~\rho^\piv(s,a) \geq 0
%	\end{aligned}
%\eeq
% \beq\label{eq:occupancy_measure}
%     \begin{aligned}
% 	&\sum_{\av \in \Abmc}\rho^\piv(s,a) - \eta(s)+\gamma
% 	%\sum_{s' \in \Smc}\sum_{\\av \in \Abmc}
% 	\sum_{\substack{s' \in \Smc \\ \av \in \Abmc}}
% 	P(s | s',a)\rho^\piv(s',a)=0 ~\\
% 	&~~~~~~~~~~~~~~~~~~~~~~~~~\rho^\piv(s,a) \geq 0
% 	\end{aligned}
% \eeq
\begin{equation}\label{eq:BFErr}
    \BFerr_f(s) = 0, \forall s \in \Smc \quad\text{AND}\quad
    f(s,\av) \geq 0, \forall s \in \Smc, \av \in \Abmc,
\end{equation}
where $\BFerr_f(s)$ denotes the {\em Bellman residual} with respect to the state-action visitation frequency:
\begin{equation*}
    \BFerr_f(s) = \sum_{\av \in \Abmc}f(s,\av) - \eta(s)-\gamma
	\sum_{s' \in \Smc}\sum_{\av \in \Abmc}
	%\sum_{\substack{s' \in \Smc \\ \av \in \Abmc}}
	P(s | s',\av)f(s',\av).
\end{equation*}
On the other direction, for an $f$ satisfying BF constraints, there is a unique stationary policy $\piv \in \pol$ associated with $f$ such that $f$ is the occupancy measure under $\piv$ (i.e., $\rho^\piv = f$); 
furthermore, such a policy can be constructed by 
\begin{align}
    \piv(s,\av)=f(s,\av) \Big/ \sum_{\av' \in \Abmc}f(s,\av'). \label{eq:occMeasToPolicy}
\end{align}
%These properties enable us 
%In the later section, we use these properties to recast optimisation problems over $\pol$ into the ones over the functions of the type $f:\Smc \times \Abmc \to \mathbb{R}$, which is more tractable.

%For an occupancy measure $\rho^\piv(s,a)$ that satisfies the equation above, $\piv(s,a)=\rho^\piv(s,a) / \sum_{\av' \in \Abmc}\rho^\piv(s,a')$ is a stationary policy of it; Also for a stationary policy $\piv(s,a)$, it has occupancy measure that satisfies $\piv(s,a)=\rho^\piv(s,a) / \sum_{\av' \in \Abmc}\rho^\piv(s,a')$ as well as the equations above, which is shown in \cite{syed2008apprenticeship}. 

\subsubsection*{Non-reward Requirements}
In addition to reward requirements captured by equilibrium concepts, non-reward requirements have also drawn attention. Here,
we consider the three typical non-reward requirements: safety, frequency and fairness requirements. These requirements can be measured as the counts of occurrences of certain states in a {\em game trajectory}, i.e., a sequence of states generated by a policy in a Markov game. A finite trajectory with length $n+1$ is written as $\tau \triangleq s_0, s_1, \ldots, s_n$. We can formalise the above-mentioned three types of non-reward requirements in a trajectory-centric way:
\begin{itemize}[leftmargin=*]
    \item \textbf{\em Safety:} For a set of undesired states $S^*$, we expect the count of undesired states in the trajectory equals zero, $\sum_{i \in [0,n]}\mathbb{I}(s_i \in S^*)=0$ where $\mathbb{I}$ is the indicator function;
    \item \textbf{\em Frequency:} For a set of specific states $S^*$, we expect the count of such states occur in trajectory with a certain proportion $c$, $\sum_{i \in [0,n]}\mathbb{I}(s_i \in S^*) / (n+1) = c$
    \item \textbf{\em Fairness:} For 2 sets of states $S^*_1, S^*_2$, we expect the counts of such states from 2 sets to be equal in trajectory, $\sum_{i \in [0,n]}\mathbb{I}(s_i \in S^*_1)=\sum_{i \in [0,n]}\mathbb{I}(s_i \in S^*_2)$.
\end{itemize}

Intuitively, we demonstrate the three types of non-reward requirements in Fig.~\ref{fig:DemoReq} on a Markov game with two different states.

\begin{figure}
 \centering
  \includegraphics[width=\columnwidth]{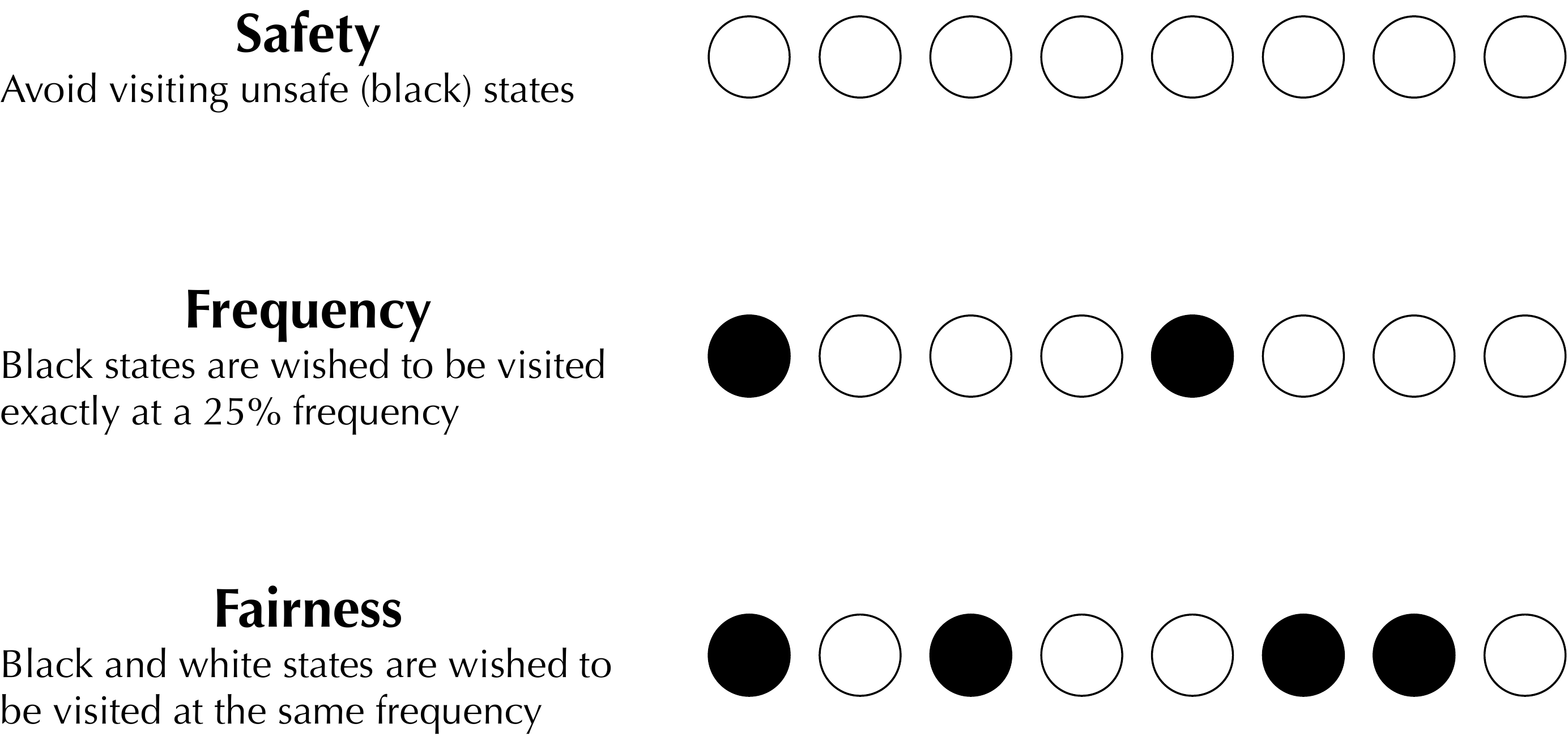}
\caption{Demonstrations of three types of non-reward requirements. Sequences of circles represent the trajectory, and black and white colors represent two different states.}
\label{fig:DemoReq}
\end{figure}

% Chapter4
\section{Density-Based Correlated Equilibria}

In this section, we first propose the general definition of {\em Density-Based Correlated Equilibria} (DBCE). We then instantiate it by specifying the selection criterion as above mentioned three types of non-reward requirements. 

Recall that the motivation of DBCE is to find an equilibrium that can capture both agents' coordination and policies' side effects
%. To this end, we define the problem as finding a specific correlated equilibrium with some given property that is related to real-world requirements but 
that cannot be simply represented in terms of rewards but can instead be interpreted using the density function. 
%We denoted 3 variants of requirements under finite trajectory in the last chapter. Here we claim that our DBCE can capture such requirements in infinite trajectories.
To this end, we formalise DBCE by taking one or a set of density functions as the selection criterion to identify the subset of CEs that satisfy the desired non-reward requirements. In such a way, DBCE is defined as a solution to a constrained optimisation problem, where the density functions serve as the objective and the constraints enforce the conditions of being a CE.

%To link the occurrence frequency of that states and the trace generated by the policy, we use density---i.e., the discounted visitation frequency---to capture the requirements because calculating the visitation under an infinite-horizon setting without discounted factor is intractable. So we define DBCE as the policy with the ability to generate the desired trajectories with corresponding state-distribution requirements. However, density is a surrogate measure of the state-distribution requirements since we use discounted factors in it. In the experiments, we empirically evaluate the ability of DBCEs on generating the desired trajectories. %(cf. RQ. ~\ref{RQ:T-score} in the experiment).

%\subsection{Solution concept}
\begin{definition}\label{def:DBCE}
%     Let the following be given: 
% 	a set of states $S^* \subseteq \Smc$, 
% 	a natural number $M$, and 
% 	a function $F(\rho^\piv(s_1),%\rho^\piv(s_2),
% 	\ldots,\rho^\piv(s_m)):\mathbb{R}^m \longrightarrow \mathbb{R}$ 
% 	defined on the density of states and a function. 
	Let the following be given: 
	\begin{itemize}
            \item A Markov Game $(\Smc, \{ \Amc_i \}_{i=1}^N, P, \{r_i\}_{i=1}^N, \eta, \gamma)$;
	    \item A subset of states $\Smc^* = \{s_1,\ldots, s_m\} \subseteq \Smc$;
            \item A real-valued function $F:\mathbb{R}^m \to \mathbb{R}$;
            \item A function $\varphi(\piv) = F(\rho^\piv(s_1),$ $\ldots,\rho^\piv(s_m))$, which we call the \emph{density error} of $\piv$.
	\end{itemize}
	A joint policy $\piv$ is called an {\em ($F$-specified) density-based correlated equilibria} (DBCE) if it is a solution to the following constrained optimisation problem:
% 		\beq\label{eq:mdce}
% 		\begin{aligned}
% 			&~~~~~~~~~~~~~~~~~~~~~~\min_{\piv} F(\rho^\piv(s_1),\rho^\piv(s_2),\ldots,\rho^\piv(s_m)) \\
% 			&\mathbb{E}_{\av_{-i} \sim \piv_{-i}(\av_{-i} \vert s, a_i)} \left[ Q_i^{\piv^{\CE}}(s, a_i', \av_{-i}) -  Q_i^{\piv^{\CE}}(s, a_i, \av_{-i}) \right] \leq 0 ~ \\
% 			&~~~~~~~~~~~~~~~~~~~~~~\forall i \in [N], s \in \Smc, a_i, a_i' \in \Amc_i.
% 		\end{aligned}
% \eeq
\begin{equation}\label{eq:CECond}
    \begin{aligned}
        \min_{\piv\in\pol} \varphi(\piv) \quad & \text{subject to}  \\
	\regret_\piv(s,i,a_i,a'_i) \leq 0, \forall i \in & [N], s \in \Smc, a_i, a_i' \in \Amc_i.
    \end{aligned}
\end{equation}
\end{definition}

We can use the value of the density error $\varphi(\piv)$ to indicate the quality of $\piv$ in terms of the state density. 
%We aim to define a suitable $\varphi$ so that a policy with the smaller error %lower (i.e., the better) $\varphi(\piv)$ generates trajectories that better satisfy requirements. 
By choosing a suitable function $F$, we can instantiate the DBCE that captures a specific non-reward requirements. Here, we introduce the following specific DBCEs with respect to safety, frequency and fairness requirements: 
%With this definition, we can define some specific forms of DBCEs. 
a DBCE is called a
\begin{itemize}
    \item \emph{Minimum Density CE (MDCE)} when we have\\ $\varphi(\piv) = \sum_{s \in \Smc^*}\rho^\piv(s)$; 
    \item {\em Frequency Matching CE (FMCE)} when we have\\ $\varphi(\piv) =  | \sum_{s \in \Smc^*}\rho^\piv(s) - c |$ for some $c\in\real_{\geq 0}$;
    \item {\em Minimum Density Gap  CE (MDGCE)} when we have\\ $\varphi(\piv) =  | \sum_{s \in \Smc_1}\rho^\piv(s) - \sum_{s \in S_2}\rho^\piv(s) |$ for $S_1, S_2 \subseteq \Smc^*$. 
\end{itemize}
Intuitively, we use MDCE, FMCE and MDGCE to represent CEs with requirements concerning safety, frequency, and fairness requirements, respectively. 
These instantiations indicate the general ability of DBCE to characterise the equilibria with some density-related properties, which are not yet able to be represented by other equilibrium notions.

% Chapter5
\section{Density-Based Correlated Equilibria Finding}

%%%%%Previous version of this file here:
%\input{Bakups/ce_finding8-14}
%%%%%

This section is devoted to the introduction of a policy iteration algorithm to compute DBCE and the proof of its convergence under certain assumptions. For simplicity, our analysis centers around the Minimum Density CE (MDCE); but it applies to any other instance of DBCE.

\subsection{Density-Based Correlated Policy Iteration}

%We want to solve the optimisation problem Def.~\ref{def:DBCE}, but we observe that the objective function is in the form of density, and the constraints are in the form of rewards. The inconsistency between reward and density makes the problem intractable. Therefore, we rewrite the problem in terms of occupancy measures to bridge the constraints and the objective function. As result, the constraints of CE and the objective function can both be represented in form of the occupancy measures, and we can solve the optimisation problem by finding occupancy measure. To involve the occupancy measure in the problem, we use the {\it Bellman flow constraints} to reach the duality between the policy and function $f: \Smc \times \Amc \rightarrow \mathbb{R}$ via Eq. (\ref{eq:occMeasToPolicy}), and $f$ is the occupancy measure once the {\it Bellman flow constraints} are satisfied.

Recall that computing an MDCE requires solving the constrained optimisation problem defined in Eq.~\eqref{eq:CECond} where $\varphi(\piv) = \sum_{s \in \Smc^*}\rho^\piv(s)$. However, directly solving it is intractable because the density functions in the objective and the expected return in the constraints are defined in two different spaces; this prevents us from optimising the density-related objective whilst satisfying rewards-related constraints in a unified fashion. We thus ask for a way to unify the representations of the state density and expected return. Fortunately, we observe that both the state density and expected return can be rewritten in terms of the occupancy measure introduced in Sec.~\ref{sec:pre}. We can thereby simultaneously control the two by maintaining a single variable, rather than in two separate spaces.

We next show how to derive an equivalent yet tractable form of the original constrained optimisation problem. We first rewrite the constraints of Eq.~\ref{eq:CECond} by occupancy measure, $\regret'_f(s,i,a_i,a'_i)$, which is defined as follows:
\[
\sum_{\av_{-i}} f(s,a_i,\av_{-i}) \left[ Q_i^{\piv}(s, a_i', \av_{-i}) -  Q_i^{\piv}(s, a_i, \av_{-i}) \right].
\]

This is equivalent to $\regret'_{\rho^\piv}(s,i,a_i,a'_i)$, shown via Eq. (\ref{eq:occMeasToPolicy}).

With the relationship between the occupancy measure and the density function, the objective function $\varphi(\piv)$ can also be rewritten as $\varphi'(f)$, where 
\begin{align}
\varphi'(f) \triangleq  F\bigl(\sum_\av f(s_1,\av),\ldots,\sum_\av f(s_m,\av)\bigr). \label{eq:varphiP}
\end{align}
Recall that the density function can be rewritten as the sum of occupancy measures within one state: $\rho^\piv(s) = \sum_{\av \in \Abmc} \rho^\piv(s,\av)$, the new objective function becomes equivalent to the original objective function. So far, we achieve the consistency between non-reward requirements and reward requirements by occupancy measure. By the properties of occupancy measure discussed in the previous section, we recast the problem as follows:

\begin{problem}\label{ProbForm2}
\(
\displaystyle 
%\underset{f:\Smc\times\Amc \to \mathbb{R}}{\mathrm{minimize}} 
\min_{f:\Smc\times\Abmc \to \mathbb{R}} 
\sum_{s \in S^*}\sum_{\av \in \Abmc}f(s,\av) 
\)
\quad {\rm subject to} 
\begin{align}
%			\min_{f:\Smc\times\Amc \to \mathbb{R}}\sum_{a \in \Amc}f(s^*,a) \quad s.t. \\
		\regret'_f(s,i,a_i,a'_i) 
			&\leq 0, &&	\forall i \in [N], s \in \Smc, a_i, a_i' \in \Amc_i; \label{eq:modifiedCECond}\\
			\BFerr_f(s) 
			&= 0, && \forall s \in \Smc; \label{eq:BFErr2}\\
            f(s,\av) 
            &\geq 0, && \forall s \in \Smc, \av \in \Abmc. \label{eq:nonneg2}
\end{align}
\end{problem}

Here, the Bellman flow constraints (\ref{eq:BFErr2}) and (\ref{eq:nonneg2}) enforce $f$ to be the occupancy measure under some $\piv \in \pol$.
For such an $f$, the new objective function $\sum_{\av \in \Abmc}f(s^*,\av)$ is equal to $\rho^\piv(s^*)$, %the density of $s^*$ under $\piv$; 
and (\ref{eq:modifiedCECond}) enforces $\piv$ to be a CE.
Due to the one-to-one correspondence between occupancy measures and stationary policies, a solution of Prob.~\ref{ProbForm2} is the occupancy measure under a solution to the original problem. 

However, Prob.~\ref{ProbForm2} is still difficult to solve directly because both $f$ and $Q$ (involved in $\regret'_f$) are unknown. 
We introduce an iterative approach to handle the problem that we call {\em Density-Based Correlated Policy Iteration} (DBCPI). It alternates between: (i) {\em policy evaluation}: estimating $Q$ values according to the current policy; and (ii) {\em policy improvement}: computing a DBCE under the current $Q$ function. More formally, let $t$ denote the index of iterations. At each iteration, $\Qv^t = \{Q^t_i\}_{i\in [N]}$ defines a {\em stage game} with constant $Q$ values.
%i.e., we iteratively solve Prob.~\ref{ProbForm2} with a fixed Q-function to gradually improve the policy.  Concretely, for $f, \myfuncG_i:\Smc\times\Amc \to \mathbb{R}$ for each $i\in [N]$ and $\myfuncBG = \{\myfuncG_i\}_{i\in [N]}$, 
Define $\regret_f^t(s,i,a_i,a'_i)$ as follows:
\[
\sum_{\av_{-i}} f(s,a_i,\av_{-i}) \left[ Q_i^t(s, a_i', \av_{-i}) -  Q_i^t(s, a_i, \av_{-i}) \right].
\]
By substituting $\regret'_f$ in Prob.~\ref{ProbForm2} with $\regret^t_f$, the stage game $\Qv^t$ is now tractable to solve 
using linear programming. After deriving the DBCE $\piv^t$ of the current stage game $\Qv^t$, we head back to update $Q$ functions and derive $\Qv^{t+1}$. The pseudocode is presented in Alg.~\ref{alg:DBCQ}.

\begin{algorithm}
%\small
   \caption{Density-Based Correlated Policy Iteration}\label{alg:DBCQ}
\begin{algorithmic}[1]
   \STATE {\bf Input:} A Markov game 
$(\Smc, \Abmc, P, \{r_i\}_{i=1}^N, \eta, \gamma)$.
% \STATE {\bf Input:} an $N$-agent Markov game 
% $(\Smc, \Abmc, P, \rv, \eta, \gamma)$, where 
% $\rv= \{r_i\}_{i=1}^N$.
   \STATE {\bf Initialisation:} $Q_i$ for each $i\in [N]$, learning rate $\alpha$
    \STATE $\piv(s,\av) \gets f(s,\av) / \sum_{\av' \in \Abmc}f(s,\av')$ %Compute a policy $\piv$ from $f$ via Eq. (\ref{eq:occMeasToPolicy})
   \FOR{each iteration}\label{line:pt1}
   \STATE $f \gets$ (solution to  Prob.~\ref{ProbForm2} with $\{Q_i\}_{i\in [N]}$)\label{line:pt5}
   		%$\myfuncBG=\{Q_i\}_{i\in [N]}$)\label{line:pt5}
   		\STATE $\piv(s,\av) \gets f(s,\av) / \sum_{\av' \in \Abmc}f(s,\av')$
       %\STATE Fix occupancy measure, derive policy $\piv$ from occupancy measure $\rho^\piv$ by $\piv(s,\av) = \rho^\piv(s,\av) / \sum_{\av' \in \Abmc}\rho^\piv(s,\av')$\\
   		\WHILE{Not converge}\label{line:pt3}
       		\STATE Initialise state $s \in \Smc$
       		\STATE Observe transition $(s,\av,\rv,s')$
       		\FOR{each $i \in [N]$}
       		    \STATE $V_i(s') \gets \sum_{\av' \in \Amc}{\piv(s',\av') Q_i(s',\av')}$
       		    \STATE $Q_i(s,\av) \gets (1-\alpha)  Q_i(s,\av) + \alpha  (r_i+\gamma  V_i(s'))$
       		\ENDFOR
       		\STATE Decay $\alpha$
   		\ENDWHILE\label{line:pt4}
   		%\STATE Fix the Q-function, solve optimisation problem \ref{ProbForm3} with current Q-function $Q$ and target states $S^*$. 
   \ENDFOR\label{line:pt2}
   \STATE {\bfseries Output:} A joint policy $\piv$, and $\varphi'(f)$ as the error of $\piv$.
\end{algorithmic}
\end{algorithm}

% \begin{algorithm}[!htbp]
% \small
%   \caption{Density Based Correlated Q Learning}\label{alg:DBCQ}
% \begin{algorithmic}[1]
%   \STATE {\bf Input:} Markov Game with parameters $(\Smc, \Abmc, N, \eta, \gamma)$, target states $S^* \subseteq \Smc$.
%   \STATE {\bf Initialisation:} Q-function $Q$, occupancy measure $\rho^\piv$, decaying learning rate $\alpha$.
   
%   \FOR{each iteration}
%       \STATE Fix occupancy measure, derive policy $\piv$ from occupancy measure $\rho^\piv$ by $\piv(s,\av) = \rho^\piv(s,\av) / \sum_{\av' \in \Abmc}\rho^\piv(s,\av')$\\
%   		\WHILE{Not converge}
%       		\STATE Initialize state $s \in \Smc$
%       		\STATE Observe transition $(s,\av,r,s')$
%       		\STATE Calculate $V(s') = \sum_{\av' \in \Amc}{\piv(s',\av')*Q(s',\av')}$
%       		\STATE Update Q-function by $Q(s,\av) := (1-\alpha) * Q(s,\av) + \alpha * (r+\gamma * V(s'))$
%       		\STATE Decay $\alpha$.
%   		\ENDWHILE
%   		\STATE Fix the Q-function, solve optimisation problem \ref{ProbForm3} with current Q-function $Q$ and target states $S^*$. 
%   \ENDFOR
%   \STATE {\bfseries Output:} Learned equilibrium $\piv$ from optimized occupancy measure $\rho^\piv$.
% \end{algorithmic}
% \end{algorithm}

\subsection{Convergence Analysis}
We next prove %the convergence of 
that $\Qv^t$ converges to the $Q$ values under a DBCE as Alg.~\ref{alg:DBCQ} is applied. %Our basic idea is to show that the iterative update of $Q$ functions will lead to a unique fixed point. 
We begin by introducing the following useful technical assumptions.

\begin{assumption}\label{assump:visit}
Each state $s \in \Smc$ and action $a_i \in \Amc_i$ for all $i \in [N]$ are visited infinitely often.
\end{assumption}

\begin{assumption}\label{assump:reward}
The reward is bounded by some constant.
\end{assumption}

\begin{assumption}\label{assump:learning_rate}
The learning rate $\alpha_t$  satisfies the following conditions: $0 \leq \alpha_t < 1$ $\forall t$, $\sum_t \alpha_t = \infty$ and $\sum_t \alpha_t^2 < \infty$.
\end{assumption}

We also need the following lemma that guarantees the policy estimation procedure after solving each stage game can converge to  $Q$ values under $\piv^t$.
\begin{lemma}[\cite{szepesvari1999unified}]\label{lem:converge}
Let $\mathbb{Q}$ be the space of all $Q$ functions. Under 
Assumption~\ref{assump:visit}-\ref{assump:learning_rate}, %~\ref{assump:visit}, \ref{assump:reward} and \ref{assump:learning_rate}, 
the iteration defined by the following converges to $Q^\piv$ with probability 1: 
\begin{equation*}
%\small
\begin{aligned}
    Q^{t+1}(s,\av) &= (1-\alpha)  Q^t(s,\av)\\
    &+ \alpha_t  \Big(r(s',\av)+\gamma \sum_{\av}{\piv(\av \vert s') Q^t(s',\av)}\Big).
\end{aligned}
\end{equation*}
\end{lemma}

%Our main theorem is built upon the following lemma.

%\begin{lemma}[Theorem 17, \cite{hu2003nash}] Under Assumptions 1-3, if the $Q$ value is unique for all optimal policies, then the $Q$ function iteratively updated in Alg.~\ref{alg:DBCQ} will converge to the unique value.

%\end{lemma}
Now, we are ready to present our main theorem which shows that  DBCPI converges to a DBCE $Q$ function under assumptions.
\begin{theorem}\label{thm:convergence}
Under Assumption~\ref{assump:visit}-\ref{assump:learning_rate}, the $Q$ function iteratively updated in Alg.~\ref{alg:DBCQ} will converge to the one under a DBCE if for all $t$, $s\in \Smc$, and $i\in [N]$, the policy $\piv_t$ is recognised as the {\em global optimum} expressed as:
\begin{equation*}\label{eq:global-optimum}
\begin{aligned}
   \forall \piv' \in \pol, \quad \mathbb{E}_{\av\sim\piv^t}[Q^t_i(s,\av)] \geq \mathbb{E}_{\av\sim\piv'}[Q^t_i(s, \av)].\\
    %\text{(2) } &\mathbb{E}_{\piv^t}[Q^t_i(s, \av)] \geq \mathbb{E}_{\av_{-i} \sim \piv^t(\cdot \vert s, a_i)}[Q^t_i(s,a_i',\av_{-i})] \text{ and }\\
    %&\mathbb{E}_{\piv^t}[Q^t_i(s, \av)] \leq \mathbb{E}_{\av_{-i}\sim \piv'(\cdot \vert s, a_i)}[Q^t_i(s,a_i,\av_{-i})],\\ 
    %& \forall a_i' \in \Amc_i, \piv' \in \pol.
\end{aligned}
\end{equation*}
\end{theorem}
\begin{proof}[Proof sketch]
The basic idea is to show that the policy guided by DBCPI monotonically improves in terms of rewards. By Lemma~\ref{lem:converge}, for all $t$, after sufficient rounds of updates, we derive $Q$ functions $\Qv^{t+1} = \{Q_i^{t+1}\}_{i\in[N]}$ under $\piv^t$. By assumption, there always exists a  globally optimal policy for at each encountered stage game $\Qv^t$. As a result, every iteration the policy monotonically improves as the iteration progresses:
    \begin{equation*}
        \mathbb{E}_{\av\sim\piv^{t+1}}[Q^{t+1}_i(s,\av)] \geq \mathbb{E}_{\av\sim\piv^t}[Q^{t+1}_i(s, \av)],
    \end{equation*}
    for all $t$, $s\in \Smc$, and $i\in [N]$. This implies that after sufficient number of 
    iterations, the policy converges to a globally optimal one, so does the $Q$ function. By solving each stage game $\Qv^t$ using linear programming, all constraints in Prob.~\ref{ProbForm2} can be satisfied. At convergence, the policy is thus a feasible DBCE. 
\end{proof}

Although Thm.~\ref{thm:convergence} tells us that the convergence holds true under strong constraints on every stage game, in experiments we find the constraint is not necessary for DBCPI to converge. This fact is in accordance with the empirical analysis in \cite{hu2003nash,yang2018mean}.

\section{Experiments}

We seek to answer the following questions via experiments:

\subsubsection*{Q1} {\em Does our algorithm find a CE better than other approaches?} \label{RQ:ourAlgFindsCE}\\
We evaluate this by checking the following values %computed via $\piv$ and $f$ 
upon the termination of Alg.~\ref{alg:DBCQ} after $K$ iterations: 
\begin{align*}
     \MaxRegret &\triangleq \max_{s,i,a_i,a'_i}\regret^{K+1}_f(s,i,a_i,a'_i), \\
     \MaxBF     &\triangleq \max_s |\BFerr_f(s)|.
\end{align*}
%``Smart Electricity Meter Data Intelligence for Future Energy Systems: A Survey''
$\MaxRegret$ can be seen as a ``distance'' between $\piv$ and the CE-set: the larger the value is, the larger incentive there exists for some agent to deviate from $\piv$.
In particular, $\piv$ is a CE when this value is non-positive. 
$\MaxBF$ evaluates the soundness of the computation in line~\ref{line:pt5}: 
the larger the value is, the further $f$ deviates from $\rho^\piv$, i.e., the occupancy measure of $\piv$.
Such a deviation of $f$ implies that the value $\varphi'(f)$ is unreliable as the error of $\piv$. 

To evaluate the algorithms in this question, we focus on the $\MaxRegret$ and $\MaxBF$ values in all cases. The smaller the values are, the better the algorithms are. 9 extra MDCE tasks are carried out to compare the $\MaxRegret$ in modified games and original games for risk-sensitive reward modification.

\subsubsection*{Q2} {\em Does our policy generate desired trajectories?}\label{RQ:T-score}\\
%\emph{Does the computed CE satisfy the state-distribution property?} 
We evaluate this by examining the patterns of individual trajectories under the policy computed by Alg.~\ref{alg:DBCQ} on all three requirements.

\subsubsection*{Q3} {\em What is the accuracy of our DBCPI compared to existing ones?}\label{RQ:D-score} \\
To make the comparison in accuracy, we compare the errors 
(i.e., the value $\varphi'(f)$ in the output of Alg.~\ref{alg:DBCQ}), and select a few instances from our data and illustrate them by plots.

\subsubsection*{Q4} {\em What is the convergence performance of our DBCPI on learning?}\label{RQ:AlgPerf}\\
We evaluate this by performing the iteration error plots of our DBCPI in the experiments.

\subsection{Experiment Setup}
\subsubsection{Game environments and tasks.}

\begin{figure}
 \centering
 \subfloat[FairGamble]{
  \includegraphics[width=0.3\columnwidth]{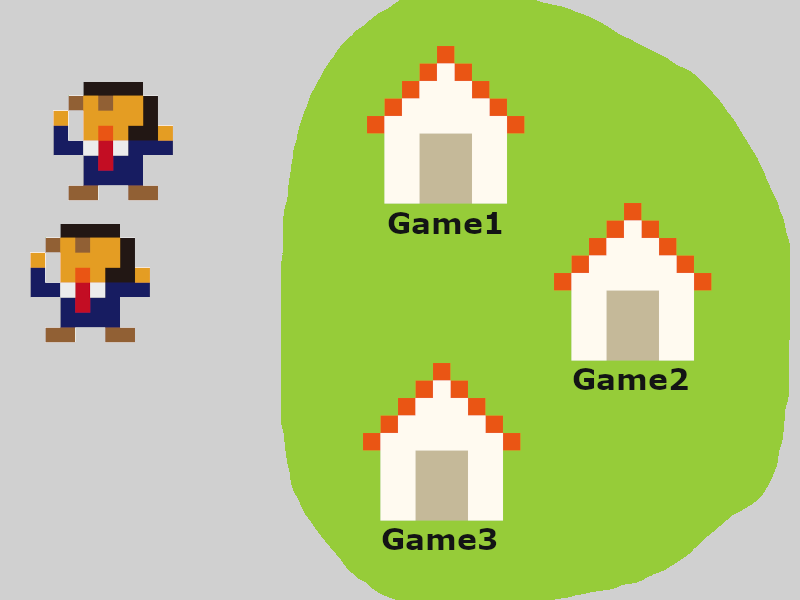}
  \label{fig:FairGamble}
 }
 \subfloat[Hunters]{
  \includegraphics[width=0.3\columnwidth]{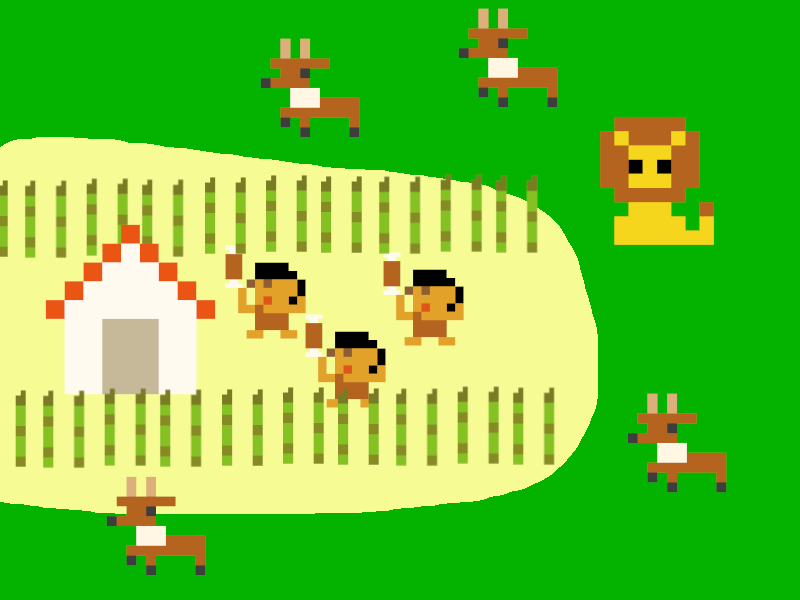}
  \label{fig:Hunters}
 }
 \subfloat[CaE]{
  \includegraphics[width=0.3\columnwidth]{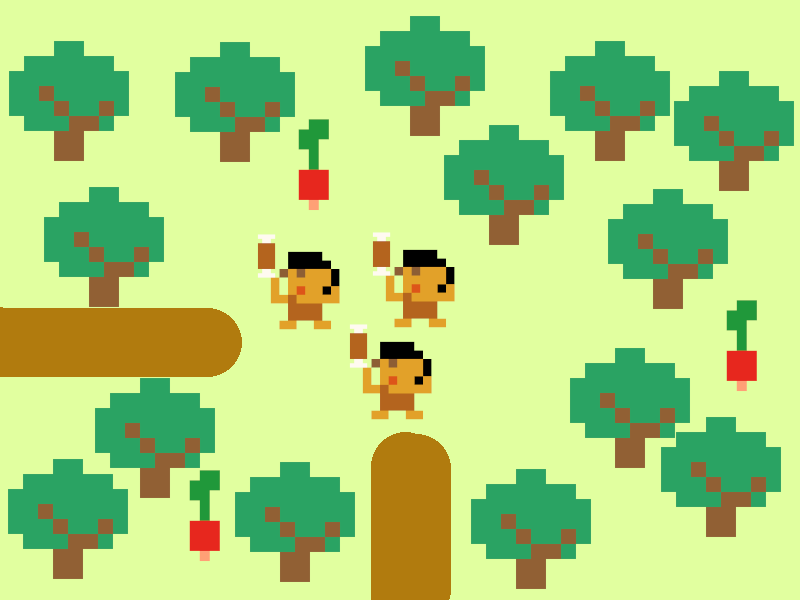}
  \label{fig:CaE}
 }
\caption{Screenshots of games.}
\label{fig:games}
\end{figure}
We consider three game models (Fair Gamble, Hunters, Collect and Explore); 
for each of them, we impose three different state-distribution requirements, which make nine instances of the input to Alg.~\ref{alg:DBCQ} in total. We have anonymously published animated demonstrations of our algorithm on these games, which are available at {\tt \url{ https://github.com/nanaralala/Density-based-Correlated-Equilibrium/}}. The screenshots are presented in Fig.~\ref{fig:games}. The descriptions of each game is shown below.
%The details of models and requirements are as follows. 
    \begin{enumerate}[leftmargin=*]
        % \item \textbf{Fair Gamble:}
        % ~\\
        \item {\it Fair Gamble}. 
         In this game, two gamblers play games with each other, and they choose from 3 different games. 3 games are extremely fair, so no matter what they do, the reward will be given randomly. Game 1 gives 0 rewards fairly; Game 2 gives 0.5 to a gambler and -0.5 to another; Game 3 gives 1 to a gambler and -1 to another.

        In each round, the gamblers choose a number from 0,1,2 and we compare the number to select which game they play. See Fig.~\ref{fig:FairGamble} for the explanation.

        We consider three state-related requirements, namely 
        \begin{enumerate}
        \item a safety requirement for gamblers to avoid game 3.
        \item a frequency matching requirement for gamblers to choose game 3 in 10\% of the time.
        \item a fairness requirement that demands game 1 and game 2 to have equal frequency.
        \end{enumerate} 
        
        % \item \textbf{Hunters:} 
        % ~\\
        \item {\it Hunters}. 
        In this non-cooperative game, 3 hunters live in one village. In each round, they are inside the village or outside the village, and they can choose between going hunting or guarding the village against the animals, see Fig.~\ref{fig:Hunters}. If one hunter goes hunting from the village, the hunter will get a high reward (1) and the rest of the hunters will get a low reward (0.1). If one hunter guards the village, all hunters will get the same mid-level reward (0.5). If one hunter is outside of the village and still stays hunting outside, we consider the behaviour is not safe enough, so the reward he gains becomes smaller (0.5), and the others will get a punishment reward (-0.5). Additionally, if there is less or equal than 1 hunter guarding the village, they will receive high punishment reward (-3).
        
        We consider three state-related requirements, namely 
        \begin{enumerate}
            \item a safety requirement that demands at least 2 hunters to stay in the village.
            \item a frequency matching requirement that requires less or equal to one agent  guarding the village 10\% of the time.
            \item a fairness requirement that demands that the frequency of hunter 1 goes hunting and the sum of hunter 2 and 3 go hunting to be equal.
        \end{enumerate} 
        % \item \textbf{Collect and Explore:}
        % ~\\
        \item {\it Collect and Explore} (CaE). 
        In this cooperative game, 3 agents are trapped in a forest, see Fig.~\ref{fig:CaE}. They can choose to explore the environment or collect some food nearby their accommodation. If more than one agent chooses to go out, we randomly choose one of them to go out, and the others will get no rewards. In each round, if existing agents go exploration, we add 1 to the reward; if any agent collects foods nearby, we add 0.3 to the reward. Since it’s a cooperative environment, we set the same reward for all agents.
        
        We consider three state-related requirements, namely 
        \begin{enumerate}
        \item a safety requirement for agent 1 to not go out.
        \item a frequency matching requirement that confines agent 1 to go out 10\% of the time.
        \item a fairness requirement that demands that the frequency of agent 1 go explore be equal to the sum of the frequencies that agent 2 and agent 3 go explore.
        \end{enumerate} 
    \end{enumerate}
    
\subsubsection{Baselines.}
    %To compare with DBCE, 
    We implement \emph{utilitarian CE-Q}~\cite{greenwald2003correlated} with risk-sensitive reward modification and constrained methods. The detail is shown as follows:
    \begin{enumerate}[leftmargin=*]
        \item \emph{Risk-sensitive Reward modification (RM).}
        A negative constant $p < 0$ is added to the reward at $s\in S^*$. 
        We write RM-$p$ to denote the algorithm with a specific $p$.
        \item \emph{Constrained method (CM).}
        An additional constraint $\varphi'(f) \leq b$ is added (cf. eq.~(\ref{eq:varphiP})). %to the Prob.~ at line~\ref{line:pt5}. %Problem~\ref{ProbForm2}.
        We write CM-$b$ to denote the algorithm with a specific $b$.
    \end{enumerate}
        RM is used as a baseline for the safety requirement only; we are not aware if there exists a canonical way to do that for frequency matching and fairness requirements. 
        CM is used for all state requirements we consider in the experiment.

\subsubsection{Implementation details.}
The iteration number of Alg.~\ref{alg:DBCQ} is set to 250. 
Parameters in Alg.~\ref{alg:DBCQ} are set to $\gamma = 0.99$ and $\alpha$ decays from 0.3 to 0.001.
We run the algorithm 3 times for each experiment environment, and the results are taken as the mean of 3 runs. 
In the program, the optimisation problem in line~\ref{line:pt5} of Alg.~\ref{alg:DBCQ} %Prob.~\ref{ProbForm2} 
was solved using an optimizer in \cite{2020SciPy-NMeth}. %Experiments are carried out on a machine with 16g RAMs and a 3.2Ghz CPU with 8 cores.

\subsection{Experiment Results}
\begin{table}[!ht]
\caption{Comparisons on capabilities of found CE.}
			\label{tab:ExpResult}
\small
		%\scriptsize
			\centering 
			\begin{tabular}{c  c  c  c  c  c}
				\toprule
				\multirow{3}{*}{Game} & \multirow{3}{*}{Metric} & \multirow{3}{*}{Method} & \multicolumn{3}{c}{Requirement}\\
				\cmidrule(r){4-6}
				& & & {\tt Safety} & {\tt Fairness} & {\tt Freq-10}\\
				\midrule 
				\multirow{12}{*}{FairGamble}
				&\multirow{4}{*}{Error}
				& DBCE & \textbf{1.225} & 32.967 & 8.683 \\
				&& CM-0.05 & 8.331 & \textbf{0.466} & \textbf{6.053} \\
				&& CM-5 & 9.65 & 8.814 & 5.668 \\
				&& RM-1.5 & 17.725 & --- & --- \\
				\cmidrule(r){2-6}
				&\multirow{4}{*}{$\MaxBF$}
				& DBCE & 0.464 & \textbf{0.13} & \textbf{0.521} \\
				&& CM-0.05 & 16.496 & 8.16 & 11.612 \\
				&& CM-5 & 18.054 & 7.781 & 6.623 \\
				&& RM-1.5 & \textbf{0.032} & --- & --- \\
				\cmidrule(r){2-6}
				&\multirow{4}{*}{$\MaxRegret$}
				& DBCE & 0.164 & \textbf{0.034} & \textbf{0.08} \\
				&& CM-0.05 & 0.174 & 0.172 & 0.22 \\
				&& CM-5 & 0.391 & 0.107 & 0.231 \\
				&& RM-1.5 & \textbf{0.11} & --- & --- \\
				\midrule 
				\multirow{12}{*}{Hunters}
				&\multirow{4}{*}{Error}
				& DBCE & 14.129 & 4.442 & 3.643 \\
				&& CM-0.05 & 2.313 & 2.283 & \textbf{0.05} \\
				&& CM-5 & 2.124 & \textbf{2.174} & 5 \\
				&& RM-1.5 & \textbf{0.828} & --- & --- \\
				\cmidrule(r){2-6}
				&\multirow{4}{*}{$\MaxBF$}
				& DBCE & \textbf{0} & \textbf{0} & 0.035 \\
				&& CM-0.05 & 0.005 & 0.001 & \textbf{0} \\
				&& CM-5 & 0.004 & 0.032 & \textbf{0} \\
				&& RM-1.5 & 0.001 & --- & --- \\
				\cmidrule(r){2-6}
				&\multirow{4}{*}{$\MaxRegret$}
				& DBCE & \textbf{0.044} & \textbf{0.061} & \textbf{0.037} \\
				&& CM-0.05 & 9.171 & 0.52 & 0.225 \\
				&& CM-5 & 0.479 & 1.18 & 0.18 \\
				&& RM-1.5 & 0.842 & --- & --- \\
				\midrule 
				\multirow{12}{*}{CaE}
				&\multirow{4}{*}{Error}
				& DBCE & 0.419 & \textbf{0.002} & 7.608 \\
				&& CM-0.05 & \textbf{0.242} & 0.05 & \textbf{1.174} \\
				&& CM-25 & 16.802 & 9.688 & 11.854 \\
				&& RM-0.5 & 23.838 & --- & --- \\
				\cmidrule(r){2-6}
				&\multirow{4}{*}{$\MaxBF$}
				& DBCE & \textbf{0} & \textbf{0} & \textbf{0} \\
				&& CM-0.05 & 0.023 & \textbf{0} & \textbf{0} \\
				&& CM-25 & 0.004 & 1.151 & \textbf{0} \\
				&& RM-0.5 & 0.023 & --- & --- \\
				\cmidrule(r){2-6}
				&\multirow{4}{*}{$\MaxRegret$}
				& DBCE & \textbf{0.002} & \textbf{0.001} & 0.003 \\
				&& CM-0.05 & 1.266 & 0.002 & 0.174 \\
				&& CM-25 & 0.419 & 6.695 & \textbf{0.001} \\
				&& RM-0.5 & 1.434 & --- & --- \\
				\bottomrule
			\end{tabular}
		\end{table}

The following results and discussions answer questions asked at the beginning of this section.

\subsubsection*{Q1.} The results for this question are found in the $\MaxBF$ and $\MaxRegret$ rows in Tab.~\ref{tab:ExpResult}. 
    DBCE has the smallest $\MaxBF$ value in 7 of 9 cases, which means DBCE performs better in finding occupancy measures and mapping them to policies. DBCE also has the smallest $\MaxRegret$ value in 7 of 9 cases, which means DBCE performs better in finding policies in the CE-set. In conclusion of this observation, our selection criteria perform better in finding CE policies in these experiments.
    In 9 extra MDCE runs, 7 of 9 runs show that the distance to CE-set in the original game is longer than the distance to CE-set in the modified game, which indicates the disadvantage of the RM method.
    
\begin{figure*}[!htbp]
 \centering
 \subfloat[FairGamble-MDCE]{
  \includegraphics[width=0.65\columnwidth]{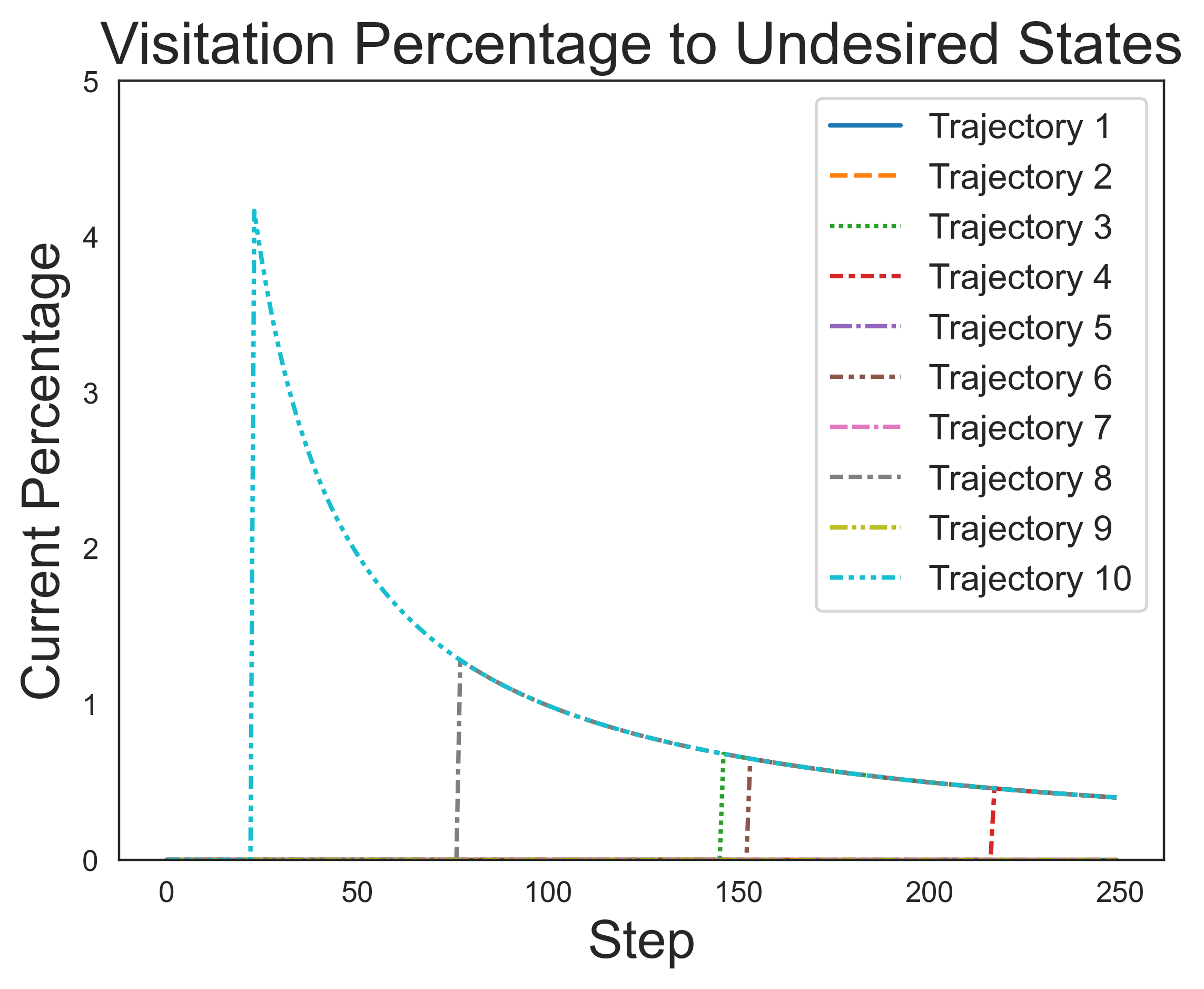}
  \label{fig:MDCETrace}
 }
 % \hspace{-8mm}
 \subfloat[Hunters-Freq-10]{
  \includegraphics[width=0.65\columnwidth]{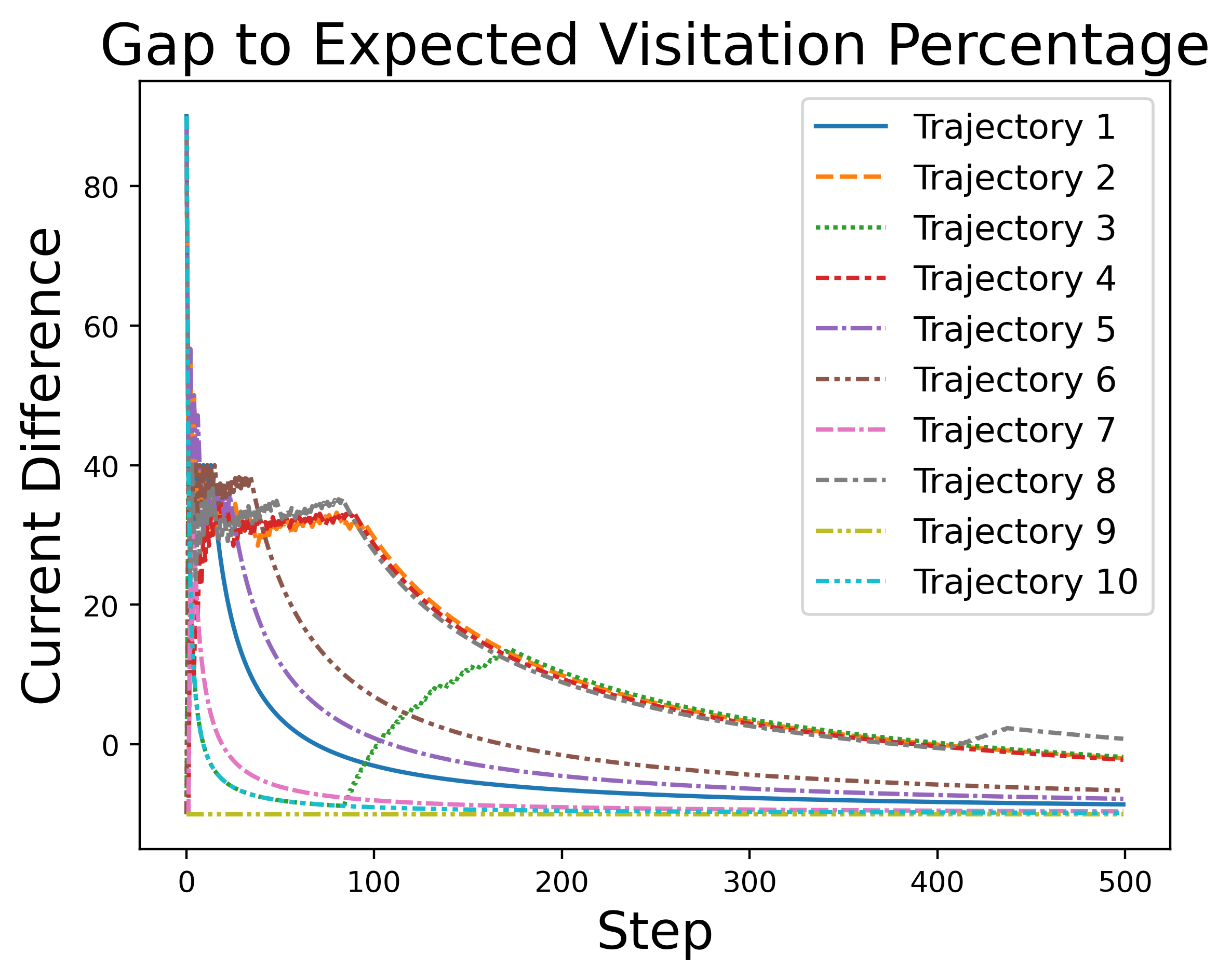}
  \label{fig:ValTrace}
 }
 % \hspace{-8mm}
 \subfloat[CaE-Fairness]{
  \includegraphics[width=0.65\columnwidth]{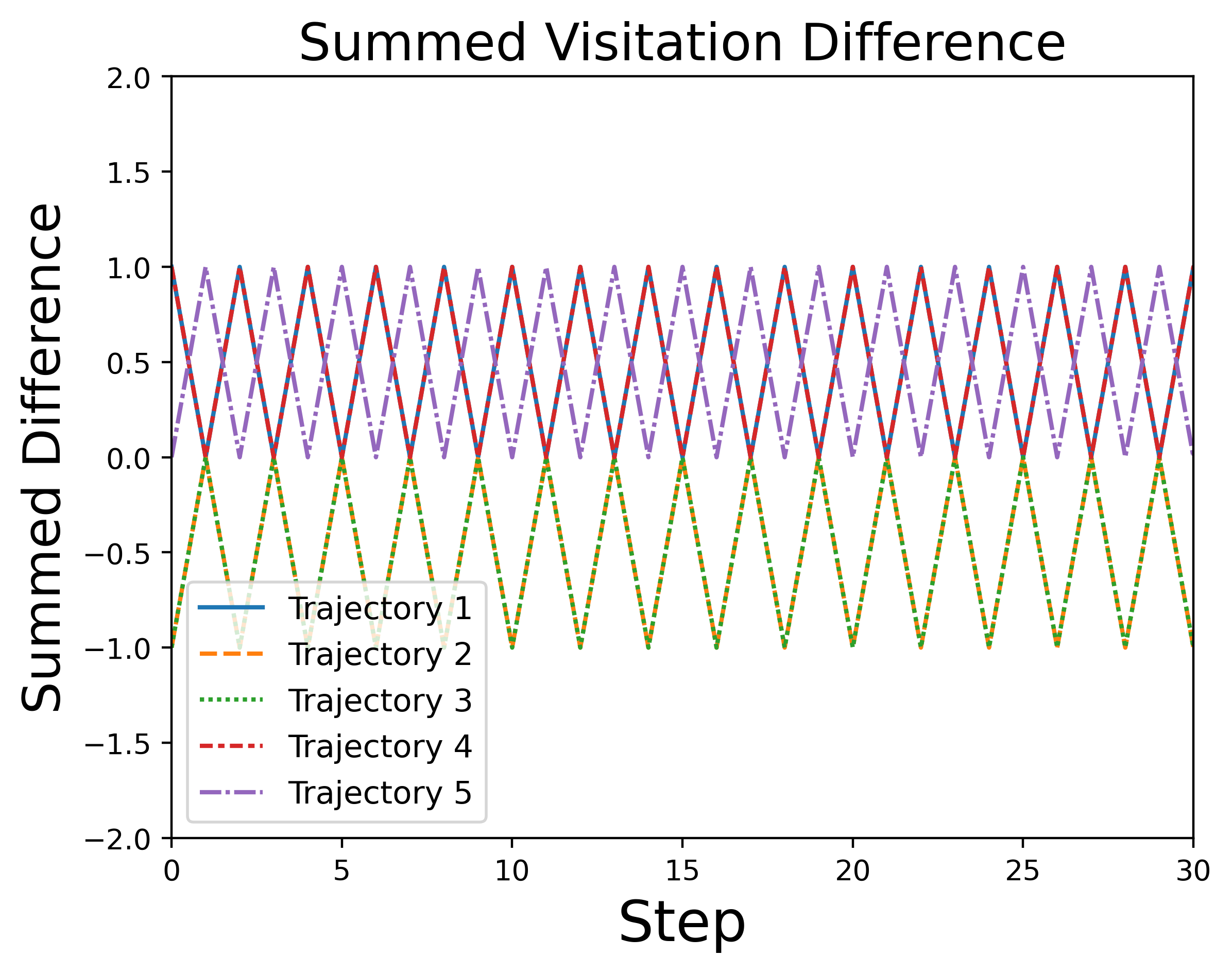}
  \label{fig:Gap}
 }
\caption{Results for trajectory performance. Close to 0 means desired trajectory.}% The trajectories are generated by our algorithm.}
\label{fig:RQ2}
\end{figure*}

\subsubsection*{Q2.} The result for this question is found in Fig.~\ref{fig:MDCETrace} for MDCE, Fig.~\ref{fig:ValTrace} for FMCE and Fig.~\ref{fig:Gap} for MGDCE. Each plot includes a few trajectories generated by our algorithm, we estimate the gap between the trajectories and the expected visitation counts, so the closer to 0 the value is, the better the trajectory is.
In Fig.~\ref{fig:MDCETrace}, we observe the trajectories have the expected property, which is the count of visitations to undesired states is low.
In Fig.~\ref{fig:ValTrace}, we observe the trajectories generated by our algorithm gradually converge to the expected proportion of visitation. Few trajectories have a larger deviation from the desired value, which indicates a larger deviation, but others perform well.
In Fig.~\ref{fig:Gap}, all trajectories perform well in balancing visitation in 30 steps (a similar pattern appears in the rest 220 steps).
   
%%%Comparison Among Different Constraint Thresholds
%\input{Tables/compConst}

\begin{figure}[!h]
    \centering
    \includegraphics[width=\columnwidth]{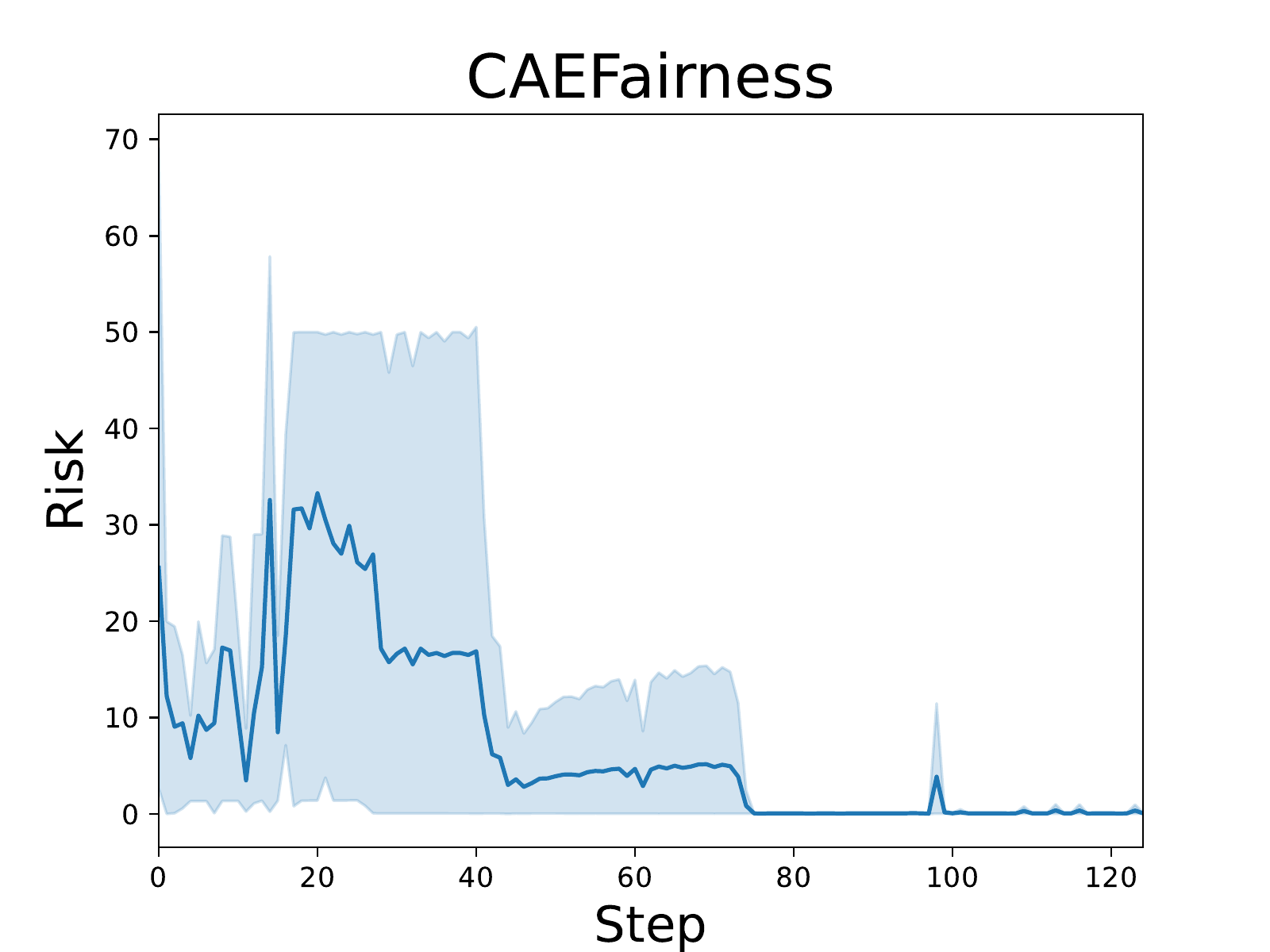}
    \caption{The error plot of our DBCE finding algorithm during learning. Take CAE-Fairness as an example. The error plot illustrates the process of the algorithm getting stabilization.}
    \label{fig:AlgPerf}
\end{figure}

\subsubsection*{Q3.} The selected result for this question is shown in Fig.~\ref{tab:RQ3}. There, we picked up the cases with comparable $\MaxRegret$ and $\MaxBF$ values, which means we drop the cases with large $\MaxRegret$ or $\MaxBF$ values;  notice that these numbers should be (close to) zero to claim that the algorithm found a CE with the computed error.
The exhaustive result is found in Tab.~\ref{tab:ExpResult}. 

In detail, DBCE has the smallest $\MaxRegret$ and $\MaxBF$ values in most cases. In {\it CaE-MinGap} case, DBCE and Cons005 have almost the same performance in $\MaxRegret$ and $\MaxBF$, but DBCE has a smaller error.

\subsubsection*{Q4.} The results for this question are found in Fig.~\ref{fig:AlgPerf}. We can observe the error-step line gradually gets stable. Fluctuations in the early stages vanish along with learning. We believe this early-stage fluctuation can be caused by the updating of Q-functions, and risk gets stable after the Q-functions get stable, which is reasonable since it's a policy-iteration process.

\begin{table}[!ht]
\caption{Explaintion on Comparable Instances}
			\label{tab:RQ3}
		%\scriptsize
  \small
			\centering 
			\begin{tabular}{c  c  c  c  c}
				\toprule
				\multirow{2}{*}{Task} & \multirow{2}{*}{Method} & \multirow{2}{*}{$\MaxRegret$} & \multirow{2}{*}{$\MaxBF$} & \multirow{2}{*}{Error} \\
				& & & & \\
				\midrule 
				\multirow{3}{*}{Hunters-MinGap}
				& DBCE & \textbf{0.061} & \textbf{0} & 4.442\\
				& Cons005 & 0.52 & 0.001 & 2.283 \\
				& Cons5 & 1.18 & 0.032 & 2.174 \\
				\midrule 
				\multirow{2}{*}{FairGamble-MDCE}
				& DBCE & 0.164 & 0.464 & 1.225 \\
				& RewMod & \textbf{0.11} & \textbf{0.032} & 17.725 \\
				\midrule 
				\multirow{3}{*}{CaE-MinGap}
				& DBCE & \textbf{0.001} & \textbf{0} & 0.002 \\
				& Cons005 & 0.002 & \textbf{0} & 0.05 \\
				& Cons5 & 6.695 & 1.151 & 9.688 \\
				\midrule 
				\multirow{4}{*}{CaE-MDCE}
				& DBCE & \textbf{0.002} & \textbf{0} & 0.419 \\
				& Cons005 & 1.266 & 0.023 & 0.242 \\
				& Cons25 & 0.419 & 0.004 & 16.802 \\
				& RewMod & 1.434 & 0.023 & 23.838 \\
				\bottomrule
			\end{tabular}
	\end{table}

\section{Conclusion and Future Works}

In this paper, we propose a new concept of the correlated equilibrium, the Density-Based correlate equilibrium (DBCE).
It enables us to find joint policies that satisfy both reward requirements, i.e., equilibrium, and non-reward requirements characterised by state density functions. Different from existing methods, DBCE neither modifies the shape or size of feasible CE-set to a game nor suffers the parameter tuning problem. We connect density and reward by occupancy measure, and design Density-Based Correlated Policy Iteration (DBCPI) to compute DBCE. Experiments on various games prove the advantage of our method in finding desired CEs.
In future works, one may be interested in implementing parameterised version of DBCPI to solve more complex games with continuous state-action space games. Additionally, one density-based objective may lead to multiple points in the CE space, so further selection among those candidates can also be the next step.

%%%%%%%%%%%%%%%%%%%%%%%%%%%%%%%%%%%%%%%%%%%%%%%%%%%%%%%%%%%%%%%%%%%%%%%%

%\bibliographystyle{aaai23}
\bibliographystyle{ACM-Reference-Format} 
\bibliography{refs}

\newpage
\newpage
\onecolumn
\section*{Appendix}

\subsubsection{More Experiment Results}

Here we show experimental results on 20 runs, to provide stronger statistical information.

\begin{table*}[ht]
\begin{tabular}{llllllllll}

\hline
Method         & Task               & \begin{tabular}[c]{@{}l@{}}Error\\ Mean\end{tabular} & \begin{tabular}[c]{@{}l@{}}Error\\ Std\end{tabular} & \begin{tabular}[c]{@{}l@{}}MaxBF\\ Mean\end{tabular} & \begin{tabular}[c]{@{}l@{}}MaxBF\\ Std\end{tabular} & \begin{tabular}[c]{@{}l@{}}MaxReg\\ Mean\end{tabular} & \begin{tabular}[c]{@{}l@{}}MaxReg\\ Std\end{tabular} & \begin{tabular}[c]{@{}l@{}}RunTime\\ Mean\end{tabular} & \begin{tabular}[c]{@{}l@{}}RunTime\\ Std\end{tabular} \\ \hline
cons005 & CaEMinGap         & 0.07      & 0.06     & 0.00      & 0.01     & 0.27       & 0.81      & 161.20      & 28.49      \\
cons25  & CaEMinGap         & 12.04     & 10.22    & 0.00      & 0.01     & 1.31       & 3.37      & 194.49      & 37.35      \\
DBCE           & CaEMinGap         & 4.27      & 8.39     & 0.05      & 0.13     & 0.31       & 0.68      & 255.07      & 42.78      \\ \hline
cons005 & CaEMDCE            & 0.48      & 0.66     & 0.05      & 0.08     & 3.49       & 7.28      & 256.63      & 33.06      \\
cons25  & CaEMDCE            & 13.93     & 10.78    & 0.70      & 3.02     & 0.58       & 1.09      & 171.41      & 27.09      \\
DBCE           & CaEMDCE            & 2.23      & 5.99     & 0.00      & 0.01     & 0.22       & 0.67      & 166.01      & 23.90      \\
ModRew         & CaEMDCE            & 20.70     & 16.81    & 0.42      & 1.75     & 0.07       & 0.13      & 208.10      & 21.92      \\ \hline
cons005 & CaEFreq-10        & 1.30      & 2.50     & 0.03      & 0.10     & 0.26       & 0.47      & 265.13      & 88.26      \\
cons25  & CaEFreq-10        & 10.84     & 5.91     & 0.00      & 0.01     & 0.07       & 0.21      & 211.06      & 33.89      \\
DBCE           & CaEFreq-10        & 11.35     & 4.78     & 0.04      & 0.14     & 0.03       & 0.05      & 203.83      & 44.49      \\ \hline
cons005 & FairGambleMinGap  & 3.07      & 4.34     & 10.84     & 8.03     & 0.18       & 0.17      & 150.41      & 20.72      \\
cons5   & FairGambleMinGap  & 8.50      & 6.00     & 6.75      & 3.85     & 0.15       & 0.07      & 177.02      & 53.26      \\
DBCE           & FairGambleMinGap  & 23.47     & 16.63    & 0.46      & 0.78     & 0.12       & 0.12      & 266.40      & 31.82      \\ \hline
cons005 & FairGambleMDCE     & 5.69      & 5.93     & 11.25     & 11.79    & 0.19       & 0.21      & 347.57      & 182.65     \\
cons5   & FairGambleMDCE     & 6.54      & 5.86     & 9.73      & 12.22    & 0.20       & 0.21      & 297.36      & 128.42     \\
DBCE           & FairGambleMDCE     & 8.05      & 14.31    & 0.19      & 0.34     & 0.15       & 0.34      & 370.30      & 36.37      \\
ModRew         & FairGambleMDCE     & 8.72      & 14.34    & 0.24      & 0.32     & 0.18       & 0.35      & 338.67      & 46.01      \\ \hline
cons005 & FairGambleFreq-10 & 5.68      & 1.57     & 13.47     & 10.93    & 0.34       & 0.30      & 130.20      & 38.12      \\
cons5   & FairGambleFreq-10 & 6.17      & 1.52     & 8.04      & 9.80     & 0.22       & 0.19      & 195.69      & 67.60      \\
DBCE           & FairGambleFreq-10 & 9.63      & 7.32     & 0.42      & 0.94     & 0.07       & 0.07      & 322.61      & 43.02      \\ \hline
cons005 & HuntMinGap        & 1.30      & 1.89     & 0.00      & 0.01     & 1.69       & 1.07      & 153.83      & 8.26       \\
cons5   & HuntMinGap        & 1.04      & 1.57     & 0.01      & 0.04     & 1.32       & 1.31      & 154.30      & 12.01      \\
DBCE           & HuntMinGap        & 2.76      & 4.35     & 0.90      & 2.94     & 0.14       & 0.35      & 397.15      & 666.03     \\ \hline
cons005 & HuntMDCE           & 1.20      & 0.73     & 0.04      & 0.13     & 4.33       & 1.78      & 525.69      & 131.39     \\
cons5   & HuntMDCE           & 2.29      & 2.44     & 0.02      & 0.07     & 1.05       & 1.50      & 162.71      & 8.36       \\
DBCE           & HuntMDCE           & 11.49     & 7.07     & 0.01      & 0.03     & 0.03       & 0.05      & 206.20      & 40.72      \\
ModRew         & HuntMDCE           & 1.67      & 1.62     & 4.31      & 18.43    & 2.81       & 1.75      & 200.18      & 14.53      \\ \hline
cons005 & HuntFreq-10       & 2.59      & 4.67     & 0.02      & 0.05     & 0.91       & 1.20      & 435.08      & 37.31      \\
cons5   & HuntFreq-10       & 5.41      & 1.24     & 0.14      & 0.47     & 0.70       & 1.40      & 257.46      & 27.31      \\
DBCE           & HuntFreq-10       & 7.69      & 5.26     & 0.00      & 0.00     & 0.07       & 0.09      & 298.59      & 39.70      \\ \hline
\end{tabular}
\end{table*}

Below we show the error plots of 2 baselines in Fig.~\ref{fig:newErrPlts}, which can be understood in a similar way as Fig.~\ref{fig:AlgPerf}.

\begin{figure*}[!htbp]
 \centering
 \subfloat[FairGamble-MDCE]{
  \includegraphics[width=0.4\columnwidth]{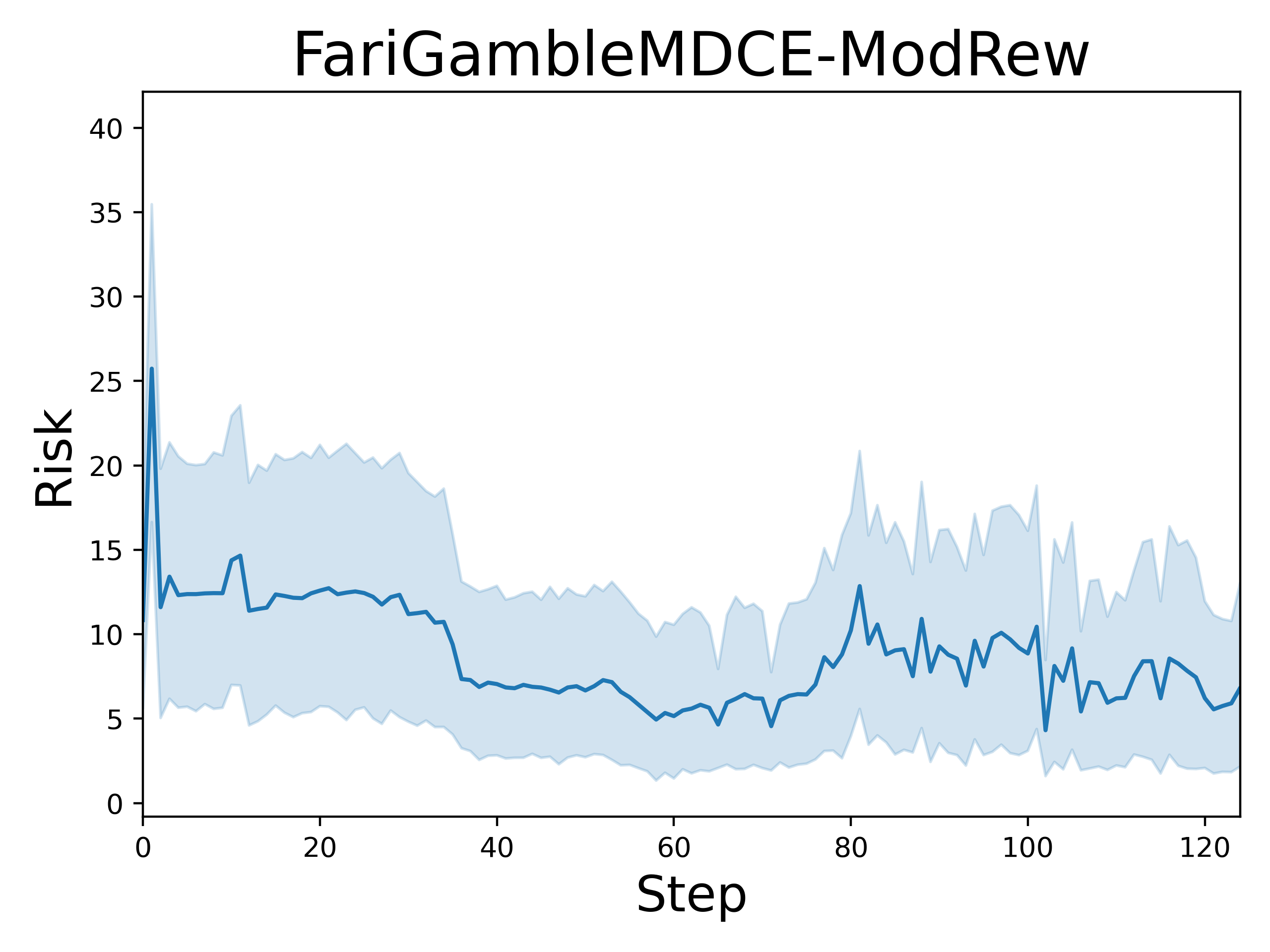}
  \label{fig:modrewErrfig}
 }
 % \hspace{-8mm}
 \subfloat[CaE-MinGap]{
  \includegraphics[width=0.4\columnwidth]{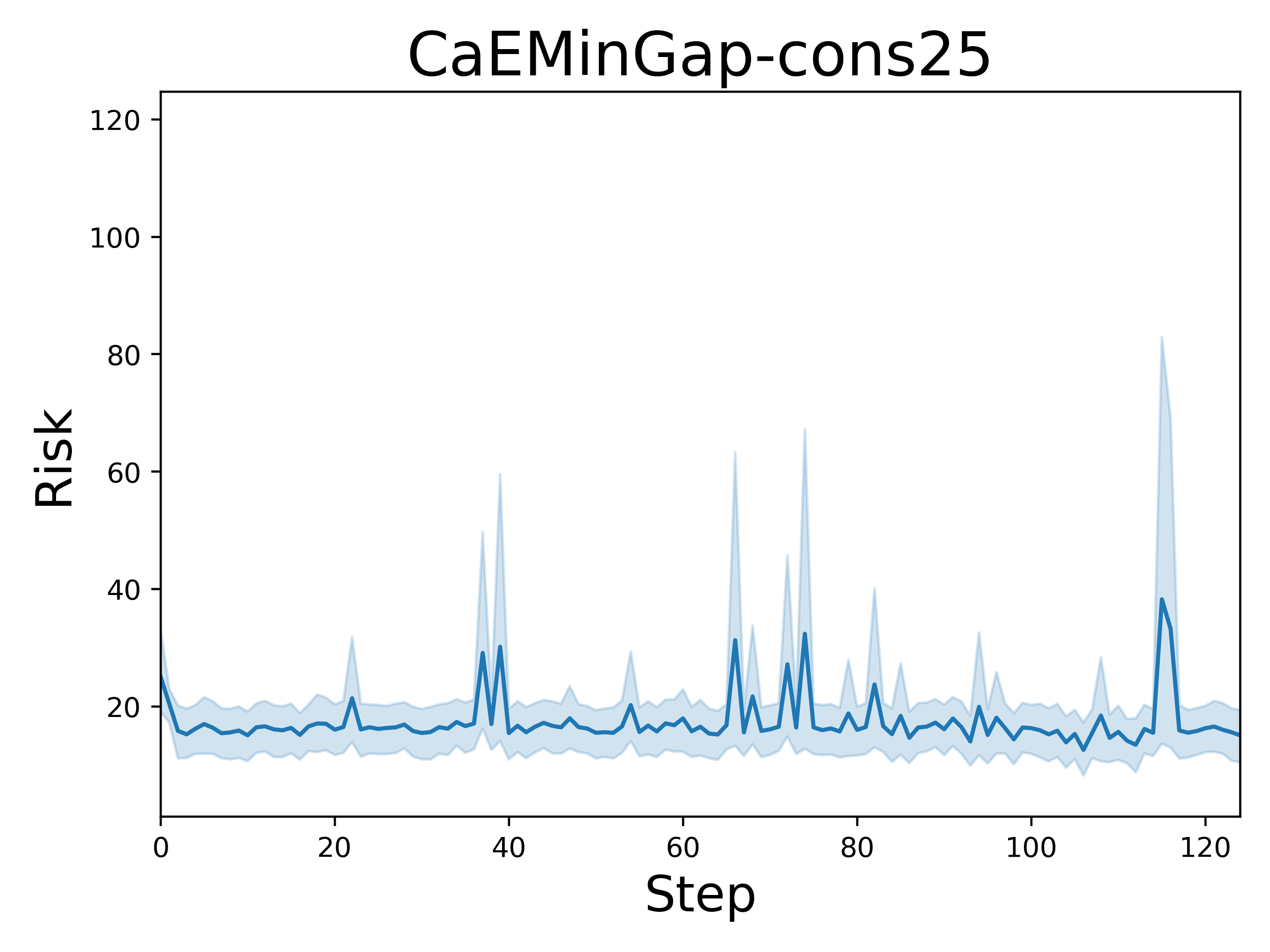}
  \label{fig:cons25Errfig}
 }
\caption{The error plot of constrained-25 and ModRew}% The trajectories are generated by our algorithm.}
\label{fig:newErrPlts}
\end{figure*}

\section{Details of Game Models}
The details of models are as follows. 

    \begin{enumerate}
        % \item \textbf{Fair Gamble:}
        % ~\\
        \item {\it Fair Gamble}. 
        In this game, two gamblers play games with each other, and they choose from 3 different games. 3 games are extremely fair, so no matter what they do, the reward will be given randomly. Game 1 gives 0 rewards fairly; Game 2 gives 0.5 to a gambler and -0.5 to another; Game 3 gives 1 to a gambler and -1 to another.

        Each round, the gamblers choose a number from 0,1,2 and we compare the number to select which game they play. See figure\ref{fig:FairGamble} for the explanation.

        % \item \textbf{Hunters:} 
        % ~\\
        \item {\it Hunters}. 
        In this game, 3 hunters live in one village. In each round, they are inside the village or outside the village, and they can choose between going hunting or guarding the village from the animals, see Figure\ref{fig:Hunters}. If one hunter goes hunting from the village, the hunter will get a high reward (1) and the rest of the hunters will get a low reward (0.1). If one hunter guards the village, all hunters will get the same mid-level reward (0.5). If one hunter is outside of the village and still hunts outside, the behaviour is not safe enough, so the reward he gains is smaller (0.5), while the others will get a punishment reward (-0.5). Additionally, if there is less or equal than 1 hunter guarding the village, they will receive high punishment reward (-3).
        
        % \item \textbf{Collect and Explore:}
        % ~\\
        \item {\it Collect and Explore}. 
        This is a cooperative game. In this environment, 3 agents are trapped in a forest, see Figure \ref{fig:CaE}. They can choose to explore the environment or collect some foods nearby their accommodation. In each round, only one agent can go out to explore the environment. If more than one agent chooses to go out, we randomly choose one of them to go out, and the others will get no rewards. In each round, if existing agents go exploration, we add 1 to the reward; if any agent collects foods nearby, we add 0.3 to the reward. Since it's a cooperative environment, we set the same reward for all agents.
        %The goal is to explore the environment, and we want to avoid the time of waiting for other agents.
        
    \end{enumerate}
\end{document}